\documentclass[letterpaper, 10 pt, conference]{format/IEEEtran}

\newif\ifdraft
\drafttrue

\ifdraft
\usepackage[letterpaper,paperwidth=9in,paperheight=11in,
			top=0.75in,bottom=1in,left=0.875in,right=0.875in,
			footskip=.25in,heightrounded,marginparwidth=0.9in,
			marginparsep=0.01in]{geometry}
\usepackage{xcolor}
\usepackage{xargs} 
\usepackage[textsize=footnotesize]{todonotes}
\newcommandx{\nt}[2][1=]{\todo[linecolor=blue,
			backgroundcolor=blue!10,bordercolor=blue,#1]{#2}}
\newcommandx{\jj}[2][1=]{\todo[linecolor=red,
			backgroundcolor=red!10,bordercolor=red,#1]{{\bf JJ}: #2}}
\def\ct#1{\textcolor{red}{#1}}
\else
\def\ct#1{}
\def\td#1{}
\def\jj#1{}
\fi

\usepackage{tabularx}
\usepackage{amssymb,amsmath,amsthm}
\usepackage[mathcal]{euscript}
\usepackage{epsfig}
\usepackage{textcomp}
\usepackage{epstopdf}
\usepackage{url}
\usepackage{wrapfig}
\usepackage{enumerate}
\usepackage{soul}
\usepackage{tipa}
\usepackage{cite}
\usepackage[percent]{overpic}
\usepackage{tikz}
\usepackage{pgfgantt}
\usepackage{xspace}
\usepackage[linesnumbered,ruled]{algorithm2e}
\usepackage{diagbox}
\usepackage{multirow}

\newtheorem{theorem}{Theorem}

\newtheorem{proposition}[theorem]{Proposition}
\newtheorem{corollary}[theorem]{Corollary}

\newtheorem{lemma}[theorem]{Lemma}
\newtheorem{problem}{Problem}

\usepackage[linesnumbered,ruled]{algorithm2e}
\SetKwInput{KwIn}{Initial condition}
\SetKwInput{KwResult}{Outcome}
\SetAlgorithmName{Algorithm}{List of algorithms}

\makeatletter
\newcommand{\customlabel}[2]{%
\protected@write \@auxout {}{\string \newlabel {#1}{{#2}{}}}}
\makeatother

\usepackage{xspace}

\def\sag{\textsc{SaG}\xspace}

\def\sagalgo{\textsc{SplitAndGroup}\xspace}

\def\mpp{MPP\xspace}
\def\tmpp{TMPP\xspace}
\def\dmpp{DMPP\xspace}

\usepackage{mathtools,xparse}

\definecolor{purp}{rgb}{0.4, 0.1, 0.6}

\font\titlefont=ptmb at 20pt
\title{{\titlefont Average Case Constant Factor Time and Distance Optimal Multi-Robot Path 
Planning in Well-Connected Environments}}
\author{\authorblockN{Jingjin Yu\authorrefmark{1}}
\authorblockA{\authorrefmark{1}Department of Computer Science, 
Rutgers University at New Brunswick, Piscataway, New Jersey, U.S.A. \\ 
jingjin.yu@cs.rutgers.edu
}
}

\begin{document}
\maketitle

\begin{abstract}Fast algorithms for optimal multi-robot path
planning are sought after in real-world applications. Known methods,
however, generally do not simultaneously guarantee good solution optimality 
and good (e.g.,  polynomial) running time. In this work, we develop a first 
low-polynomial running time algorithm, called \sagalgo(\sag), that solves 
the multi-robot path planning problem on grids and grid-like environments, 
and produces constant factor makespan optimal solutions on average over 
all problem instances. 
That is, \sag is an average case $O(1)$-approximation algorithm and
computes solutions with sub-linear makespan. \sag is capable of handling 
cases when the density of robots is extremely high - in a graph-theoretic 
setting, the algorithm supports cases where all vertices of the underlying 
graph are occupied. \sag attains its desirable properties through 
a careful combination of a novel divide-and-conquer technique, which we denote 
as {\em global decoupling}, and network flow based methods for routing the 
robots. Solutions from \sag, in a weaker sense, are also a constant factor 
approximation on total distance optimality. 
\end{abstract}

\section{Introduction}\label{section:introduction}
\begin{figure}[h]
  \begin{overpic}[width=\columnwidth]{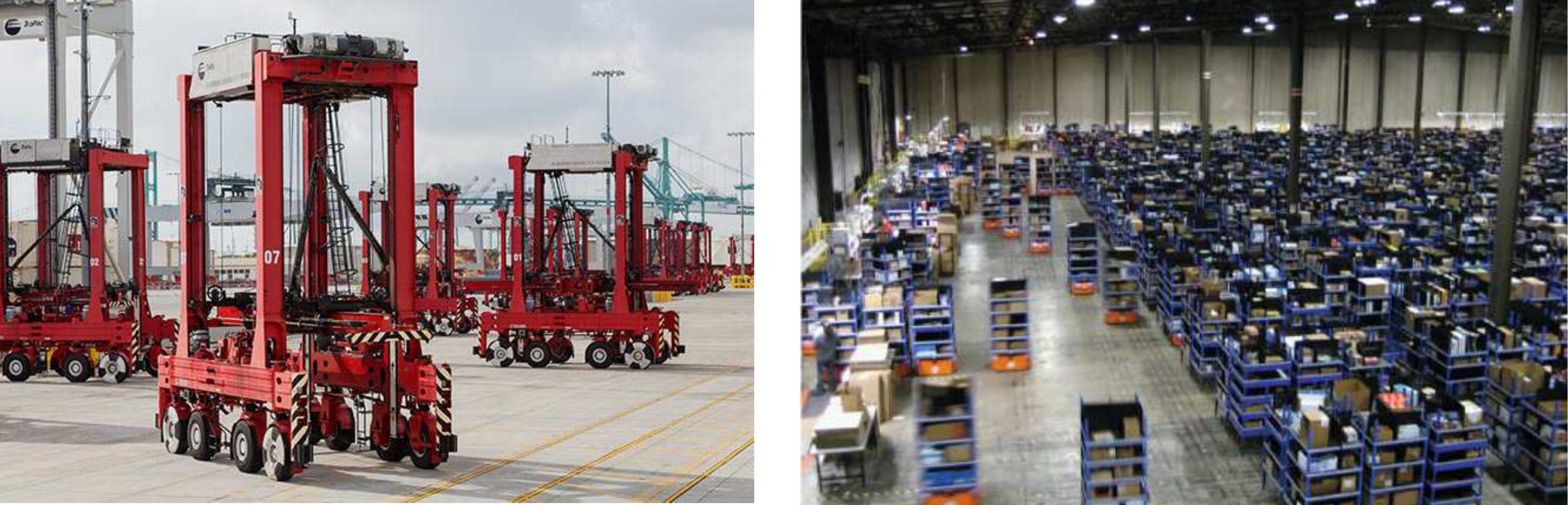}
	\put(22,-4.5){{\small (a)}}
	\put(72,-4.5){{\small (b)}}
	\end{overpic}
	\vspace*{0mm}
  \caption{\label{figure:kiva} (a) Automated straddle carriers at the 
	port of Los Angeles. Each straddle carrier is capable of autonomously 
	navigate to pick up or drop off a shipping container at a designated 
	location. (b) Amazon's Kiva multi-robot system working at its order
	fulfillment centers.}
	\vspace*{-2mm}
\end{figure}Fast methods for multi-robot path planning have found many real-world
applications including shipping container handling (Fig.~\ref{figure:kiva}(a)), 
order fulfillment (Fig.~\ref{figure:kiva}(b)), horticulture, among others, 
drastically improving the associated process efficiency. While commercial 
applications have been able to scale quite well, e.g., a single Amazon 
fulfillment center can operate thousands of Kiva mobile robots, it 
remains unclear what level of optimality is achieved by the underlying 
planning and scheduling algorithms in these applications. As such, there 
remains the opportunity of uncovering structural insights and novel algorithmic 
solutions that substantially improve the throughput of current production 
systems that contain some multi-robot (or more generally, multi-body, where a single 
body may be actuated using some external mechanism) sub-systems. 

The disconnection that exists in applications regarding multi-robot routing may 
be more formally characterized as an {\em optimality-efficiency gap} that has 
been outstanding in the multi-robot research domain for quite some time: 
known algorithms for multi-robot path planning do not simultaneous guarantee 
good solution optimality and fast running time. This is not entirely surprising 
as it is well known that optimal multi-robot path planning problems are 
generally NP-hard \cite{goldreich2011finding,Yu2015IntractabilityPlanar,banfi2017intractability}. 
Nevertheless, whereas these negative results suggest that finding polynomial-time
algorithms that compute exact optimal solutions for multi-robot path planning 
problems is impossible, they do not preclude the existence of polynomial-time 
algorithms that compute approximately optimal solutions. 

Motivated by both practical relevance and theoretical significance, 
in this work, we narrow this optimality-efficiency gap in multi-robot 
path planning, focusing on a class of {\em grid-like, well-connected} 
environments. Here, well-connected environments (to be formally defined) include 
the container shipping port scenario and the Amazon fulfillment 
center scenario. A key property of these environments is that sub-linear 
time-optimal solution is possible, which is not true for general 
environments. Using a careful combination of divide-and-conquer and
network flow techniques, we show that {\em constant factor makespan optimal} 
solutions can be computed in {\em low-polynomial running time} in the 
{\em average case}, where the average is computed over all possible problem instances for 
an arbitrary fixed environment. We call the resulting algorithm \sagalgo(\sag). In 
other words, \sag can efficiently compute $O(1)$-approximate solutions 
on average. The current paper is devoted to establishing the construction,
correctness, and key properties of \sag; we refer readers to \cite{HanRodYu18IROS}
for implementations and performance characteristics of \sag, where these 
issues are examined in detail.  

Intuitively, when the density of the robots are high in a given environment, 
computing solutions for optimally routing these robots will be more difficult. 
With this line of research, our ultimate goal is to achieve a fine-grained 
structural understanding of the multi-robot path planning problem that allows 
the design of algorithms to gracefully balance between robot density and 
computational efficiency. As we know, when the density of robots are low, 
planning can be rather trivial: paths may be planned for individual robots first
and because conflicts are rare, they can be resolved on the fly. In the current 
work, we attack the other end of the spectrum: we focus on the case of having 
a robot occupy each vertex of the underlying discrete graph, i.e., we work with 
the case of highest possible density under the given formulation. Beside the 
obvious theoretical challenge that is involved, we believe the study 
benefits algorithm design for lower density cases. Regarding this, a particularly
interesting tool developed in this work is the {\em global decoupling} technique that enables the 
\sag algorithm. 

\textbf{Contributions}. The main contribution brought forth by this work is 
a first low-polynomial time, deterministic algorithm, \sag, for solving the 
optimal multi-robot path planning problem on grids and grid-like, well-connected 
environments. Under the prescribed settings, \sag computes a solution with 
sub-linear makespan. Moreover, the solution is only a constant multiple
of the optimal solution on average. In a weaker sense, \sag also 
computes solutions with total distance a constant multiple of the optimal for 
a typical instance on average. The results presented in this work expand over 
a conference publication \cite{yu2017expected}. Most notably, this paper 
{\em (i)} provides a fuller account of the motivation and relevance that 
underlie the work, covering both practical and theoretical aspects, and 
{\em (ii)} includes complete proofs for all theorems; many of these proofs are 
much improved versions that are more clear than what appeared (as sketches) in 
\cite{yu2017expected}.

\textbf{Organization}. The rest of the paper is organized as follows. 
Related works are discussed in Sec.~\ref{section:related}. In 
Sec.~\ref{section:formulation}, the discrete multi-robot path planning 
problem is formally defined, followed by analysis on connectivity for 
achieving good solution optimality. This leads us to the choice of 
grid-like environments. We describe the details of \sag
in Sec.~\ref{section:makespan-algorithm}. In Sec.~\ref{section:complexity}, 
complexity and optimality properties of \sag are established. In 
Sec.~\ref{section:general}, we show that \sag generalizes to higher 
dimensions and (grid-like) well-connected environments including continuous
ones. 

\section{Related Work}\label{section:related}
In multi-robot path and motion planning, the main 
goal is for the moving bodies, e.g., robots or vehicles, to reach their 
respective destinations, collision-free. Frequently, certain optimality 
measure (e.g., time, distance, communication) is also imposed. Variations 
of the multi-robot path and motion planning problem have been actively 
studied for decades 
\cite{ErdLoz86,LavHut98b,GuoPar02,Sil05,Rya08,JanStu08,LunBer11,StaKor11,
BerSnoLinMan09,SolHal12,YuLav13STAR,TurMicKum14,ChoLynHutKanBurKavThr05,
blm-rvo,branicky2006sampling,khatib1986real,earl2005iterative,
bekris2007decentralized,knepper2012pedestrian,alonso2015local}. As a 
fundamental problem, it finds applications in a diverse array of areas 
including assembly \cite{HalLatWil00,Nna92}, evacuation \cite{RodAma10}, 
formation \cite{BalArk98,PodSuk04,ShuMurBen07,SmiEgeHow08,TanPapKum04}, 
localization \cite{FoxBurKruThr00}, micro droplet manipulation 
\cite{GriAke05}, object transportation
\cite{MatNilSim95,RusDonJen95}, and search-rescue \cite{JenWheEva97}. 
In industrial applications pertinent to the current work, {\em centralized}
planners are generally employed to enforce global control to drive operational 
efficiency. The algorithm proposed in this work also follows this paradigm. 

Similar to single robot problems involving potentially many degrees of 
freedom \cite{reif1985complexity,Can88}, multi-robot path planning is 
strongly NP-hard even for discs in simple polygons \cite{SpiYak84} and 
PSPACE-hard for translating rectangles \cite{HopSchSha84}. The hardness of 
the problem extends to unlabeled case \cite{KloHut06} where it remains 
highly intractable \cite{HeaDem05,SolHal15}. Nevertheless, under appropriate 
settings, the unlabeled case can be solved near optimally 
\cite{KatYuLav13ICRA-C,TurMicKum14,adler2015efficient,SolYu15}.

Because general (labeled) optimal multi-robot path planning problems in 
continuous domains are extremely challenging, a common approach is to start 
with a discrete setting from the onset. Significant progress has been made on 
solving the problem optimally in discrete settings, in particular on grid-based 
environments. Multi-robot motion planning is less computationally expensive in 
discrete domains, with the feasibility problem readily solvable in $O(|V|^3)$ 
time, in which $|V|$ is the number of vertices of the discrete graph where the 
robots may reside \cite{AulMonParPer99,GorHas10,YuArxiv-1301-2342,YuRus15STAR}. 
In particular, \cite{YuRus15STAR} shows that the setting considered 
in this paper is always feasible except when the grid graph has only four 
vertices (which is a trivial case that can be safely ignored). 

Optimal versions of the problem remain computationally intractable in a 
graph-theoretic setting \cite{goldreich2011finding,RatWar90,YuLav13AAAI,
Yu2015IntractabilityPlanar,banfi2017intractability}, but the complexity has dropped from PSPACE-hard 
to NP-complete in many cases. This has allowed the application of the 
intuitive decoupling-based heuristics \cite{alami1995multi,qutub1997solve,saha2006multi} 
to address several different costs. In \cite{StaKor11}, individual paths
are planned first. Then, interacting paths are grouped together for which 
collision-free paths are scheduled using Operator Decomposition (OD). The 
resulting algorithm can also be made complete (i.e., an anytime algorithm). 
Sub-dimensional expansion techniques (M*) were used in \cite{WagChoC11,ferner2013odrm}
that actively restrict the search domain for groups of robots. Conflict Based 
Search (CBS) \cite{ShaSteFelStu12,sharon2013increasing,boyarski2015icbs} 
maintains a constraint tree (CT) for facilitating its search to resolve 
potential conflicts. With \cite{cohen2016improved}, efficient algorithms are 
supplied that compute solutions with bounded optimality guarantees. Robots with 
kinematic constraints are dealt with in \cite{honig2016multi}. Beyond decoupling, 
other ideas have also been explored, including casting the problem as other 
known NP-hard problems \cite{Sur12,erdem2013general,YuLav16TRO} for which 
high-performance solvers are available. More recently, robustness, longer 
horizon, and other related issues have been studied in detail 
\cite{ma2017lifelong,atzmon2018robust}.

\section{Preliminaries}\label{section:formulation}

In this section, we state the multi-robot path planning problem 
and two important associated optimality objectives, in a graph-theoretic 
setting. Then, we show that working with arbitrary graphs may lead 
to rather sub-optimal solutions (i.e., super-linear with respect to the 
number of vertices). This necessitates the restriction of the graphs if 
desirable optimality results are to be achieved. 

\subsection{Graph-Theoretic Optimal Multi-Robot Path Planning}
Let $G = (V, E)$ be a simple, undirected, and connected graph. A set of $N$ 
labeled robots may move synchronously in a collision-free manner on $G$. At 
integer {\em time steps} starting from $t = 0$, each robot resides on a unique 
vertex of $G$, inducing a {\em configuration} $X$ of the robots. Effectively, 
$X$ is an injective map $X: \{1, \ldots, N\} \to V$ specifying which robot 
occupies which vertex (see Fig.~\ref{fig:problem}). From step $t$ to 
step $t + 1$, a robot may {\em move} from its current vertex to an 
adjacent one under two collision avoidance conditions: {\em (i)} the new 
configuration at $t+ 1$ remains an injective map, i.e., each robot occupies 
a unique vertex, and {\em (ii)} no two robots may travel along the same edge 
in opposite directions.

\begin{figure}[h]
\begin{center}
\begin{overpic}[width=0.5\textwidth,tics=5]{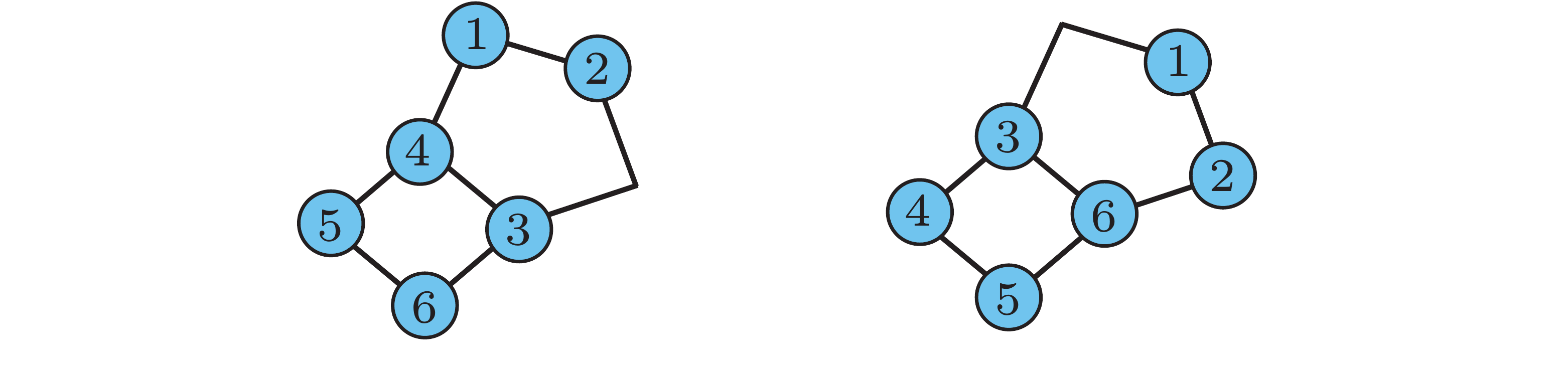}
\put(31, -3){{\small (a)}}
\put(68, -3){{\small (b)}}
\end{overpic}
\end{center}
\caption{\label{fig:problem} Graph-theoretic formulation of the multi-robot 
path planning problem. (a) A configuration of six robots on a graph (roadmap) 
with seven vertices. (b) A configuration that is reachable from (a) in a single 
synchronous move.}  
\end{figure}

A multi-robot path planning problem (\mpp) is fully defined by a 3-tuple 
$(G, X_I, X_G)$ in which $X_I$ and $X_G$ are two configurations. In this work, 
we look at the most constraining case of $|X_I| = |X_G| = |V|$. That is, all 
vertices of $G$ are occupied. We are interested in two optimal \mpp 
formulations. In what follows, {\em makespan} is the time span covering the 
start to the end of a task. All edges of $G$ are assumed to have a length 
of $1$ so that a robot traveling at unit speed can cross it in a single time 
step.

\begin{problem}[Minimum Makespan (\tmpp)] Given $G, X_I$, and $X_G$, compute
a sequence of moves that takes $X_I$ to $X_G$ while minimizing the makespan.
\end{problem}

\begin{problem}[Minimum Total Distance (\dmpp)] Given $G, X_I$, and $X_G$, 
compute a sequence of moves that takes $X_I$ to $X_G$ while minimizing the 
total distance traveled.
\end{problem}

These two problems are NP-hard and cannot always be solved 
simultaneously \cite{YuLav13AAAI}.
\subsection{Effects of Environment Connectivity}
The well-known {\em pebble motion} problems, which are highly similar 
to \mpp, may require $\Omega(|V|^3)$ individual moves to solve \cite{Kor84}.
Since each pebble (robot) may only move once per step, at most $|V|$ individual 
moves can happen in a step. This implies that pebble motion problems, 
even with synchronous moves, can have an optimal makespan of $\Omega(|V|^2)$, 
which is super linear (i.e. $\omega(|V|)$). The same is true for \tmpp under 
certain graph topologies. We first prove a simple but useful lemma for a class 
of graphs we call {\em figure-8} graphs. In such a graph, there are 
$|V| = 7n + 6$ vertices for some integer $n \ge 0$. The graph is formed by three 
disjoint paths of lengths $n, 3n+2$, and $3n+2$, meeting at two common end 
vertices. Figure-8 graphs with $n = 1$ are illustrated in Fig.~\ref{fig:exchange}.

An interesting property of figure-8 graphs is that an arbitrary MPP instance
on such a graph is feasible. 
%\jj{can omit the proof and point to the pebble paper}
\begin{lemma}\label{l:f8-feasible}An arbitrary \mpp instance $(G, X_I, X_G)$ 
is feasible when $G$ is a figure-8 graph. 
\end{lemma}
\begin{proof}
Using the three-step plan provided in Fig.~\ref{fig:exchange}, we may 
exchange the locations of robots $1$ and $2$ without collision. This 
three-step plan is scale invariant and applies to any $n$. With the 
three-step plan, the locations of any two adjacent robots (e.g., robots 
$4$ and $5$ in the top left figure of Fig.~\ref{fig:exchange}) can be 
exchanged. To do so, we may first rotate the two adjacent robots of interest 
to the locations of robots $1$ and $2$, do the exchange using the three-step 
plan, and then reverse the initial rotation. Let us denote such a sequence 
of moves as a {\em 2-switch} (more formally known as a {\em transposition} 
in group theory). Because the exchange of any two robots on the 
figure-8 graph can be decomposed into a sequence of 2-switches, such exchanges 
are always feasible. As an example, the exchange of robots $4$ and $9$ can 
be carried out using a 2-switch sequence $\langle (3, 4), (2, 4), (1, 4), 
(4, 9), (1, 9), (2, 9), (3, 9)\rangle$, of which each individual pair consists
of two adjacent robots after the previous 2-switch is completed. Because solving 
the \mpp instance $(G, X_I, X_G)$ can be always decomposed into a sequence of 
two-robot exchanges, arbitrary \mpp instances are solvable on figure-8 
graphs. ~\qed
\end{proof}
\begin{figure}[h]
\begin{center}
\begin{overpic}[width=0.5\textwidth,tics=5]{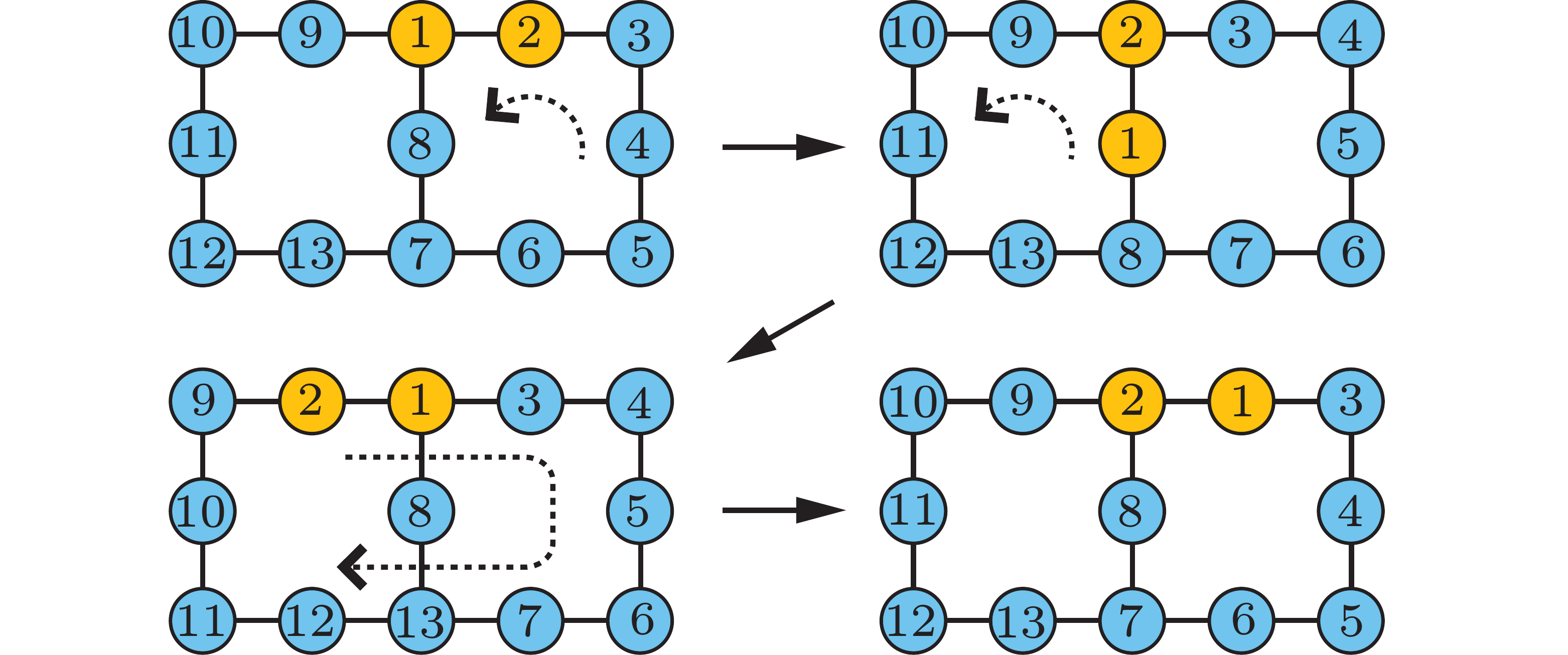}
\end{overpic}
\end{center}
\caption{\label{fig:exchange} A three-step plan for exchanging robots 
$1$ and $2$ on a figure-8 graph with $7n + 6$ vertices ($n = 1$ in 
this case).}  
\end{figure}

The introduction of figure-8 graphs allows us to formally establish that 
sub-linear optimal solutions are not possible on an arbitrary connected
graph. 

\begin{proposition}\label{t:makespan-lower-bound}There exists an infinite family 
of \tmpp instances on figure-8 graphs with $\omega(|V|)$ minimum makespan. 
\end{proposition}
\begin{proof}We will establish the claim on the family of figure-8 graphs. 
By Lemma~\ref{l:f8-feasible}, there exists a sequence of moves that takes
arbitrary configuration $X_I$ to arbitrary configuration $X_G$. For a 
figure-8 graph with $|V|$ vertices, there are $|V|!$ possible 
configurations. Starting from an arbitrary configuration $X_I$, let us 
build a tree of adjacent configurations (two configurations are adjacent if 
a single move changes one configuration to the other) with $X_I$ as the root 
and estimate its height $h_T$, which bounds the minimum possible makespan. In 
each move, only one of the three cycles on the figure-8 graph may be used to 
move the robots and each cycle may be moved in clockwise or counterclockwise 
direction; no two cycles may be rotated simultaneously. Therefore, the tree has 
a branching factor of at most $6$. Assume the best case in which the tree 
is balanced and has no duplicate nodes (i.e., configuration), we can bound 
$h_T$ as $6^{h_T+1} \ge |V|!$. That is, the tree must have at least $|V|!$ 
unique configuration nodes derived from the root $X_I$, because all $|V|!$ 
configurations are reachable from $X_I$. With Stirling's approximation 
\cite{bollobas2013modern}, 
\makeatletter
\setlength{\@mathmargin}{5.5em}
\makeatother
\[|V|! \ge \sqrt{2\pi|V|}(\frac{|V|}{e})^{|V|},\]
which yields 
\makeatletter
\setlength{\@mathmargin}{6em}
\makeatother
\[h_T = \Omega(|V|\log|V|).\]
This shows that solving some 
instances on figure-8 graphs requires $\Omega(|V|\log|V|)$ steps, 
establishing that \tmpp could require a minimum makespan of $\omega(|V|)$. 

Because $n$ in the figure-$8$ graph is an arbitrary non-negative 
integer, $|V|$ has an infinite number of values. Hence, there is an 
infinite family of such graphs. 
~\qed
\end{proof}

Proposition~\ref{t:makespan-lower-bound} implies that if the classes of graphs
are not restricted, we cannot always hope for the existence of solutions 
with linear or better makespan with respect to the number of vertices of the 
graphs, i.e., 

\begin{corollary}\label{c:no-general-linear}\tmpp does not admit solutions 
with linear or sub-linear makespan on an arbitrary graph.  
\end{corollary}

Corollary~\ref{c:no-general-linear} suggests that seeking general algorithms 
for providing linear or sub-linear makespan that apply to all environments 
will be a fruitless attempt. With this in mind, the paper mainly focuses
on a restricted but very practical class of discrete environments: grid graphs.

\section{Routing Robots on Rectangular Grids with a Sub-Linear Makespan}
\label{section:makespan-algorithm}
\subsection{Main Result}
We first outline the main algorithmic result of this work and the key enabling 
idea behind it, a divide-and-conquer scheme which we denote as {\em global 
decoupling}.

Assuming unit edge lengths, 
a rectangular grid is fully specified by two integers $m_{\ell}$ and $m_s$, 
representing the number of vertices on the long and short sides of the 
 grid, respectively. Without loss of generality, assume that 
$m_{\ell} \ge m_s$ (see Fig.~\ref{fig:sketch} for a $8 \times 4$ grid). 
We further assume that $m_{\ell} \ge 3$ and $m_s \ge 2$ since an \mpp on 
a smaller grid is trivial. These assumptions are implicitly assumed in this 
paper whenever {\em grid} is mentioned, unless otherwise stated. We note that
an \mpp problem on such a grid is always feasible, as established formally in 
\cite{YuRus15STAR}. The main result to be proven in this section is the 
following. 
\begin{theorem}\label{t:grid-linear-makespan}
Let $(G, X_I, X_G)$ be an arbitrary \tmpp instance in which $G$ is an 
$m_{\ell} \times m_s$ grid. The instance admits a solution with 
$O(m_{\ell})$ makespan. 
\end{theorem}

Note that the $O(m_{\ell})$ bound is {\em sub-linear} with respect to the 
number of vertices, which is $\Omega(m_sm_{\ell})$ and $\Omega(m_{\ell}^2)$ 
for square grids. We name the algorithm, to be constructed, as \sagalgo 
(\sag) and first sketch how the divide-and-conquer algorithm works at a 
high level. In this section we focus on the makespan property  
of \sag. We delay the establishment of polynomial-time complexity and 
additional properties of the algorithm to Section~\ref{section:complexity}. 

Assume without loss of generality that $m_{\ell} = 2^{k_1}$ $m_s = 2^{k_2}$ 
for some integers $k_1$ and $k_2$ (we note that our algorithm does not depend on 
$m_{\ell}$ and $m_s$ being powers of $2$ at all; the assumption only serves to 
simplify this high-level explanation). 
In the first iteration of \sag, it {\em splits} the grid into two smaller 
rectangular grids, $G_1$ and $G_2$, of size $2^{k_1-1} \times 2^{k_2}$ each. 
Then, robots are moved so that at the end of the iteration, if a robot has its
goal in $G_1$ (resp., $G_2$) in $X_G$, it should be on some arbitrary vertex of 
$G_1$ (resp., $G_2$). This is the {\em group} operation. An example of a single \sag 
iteration is shown in Fig.~\ref{fig:sketch}. We will show that such an 
iteration can be completed in $O(m_{\ell}) = O(2^{k_1})$ steps (makespan). 
In the second iteration, the same process is carried out on both $G_1$ and 
$G_2$ in parallel, which again requires $O(m_{\ell}) = O(2^{k_1})$ steps. 
In the third iteration, we start with four $2^{k_1-1} \times 2^{k_2-1}$ 
grids and the iteration can be completed in $O(2^{k_1-1}) = 
O(\frac{m_{\ell}}{2})$ steps. After $2k_1$ iterations, the problem is solved 
with a makespan of 
\makeatletter
\setlength{\@mathmargin}{1em}
\makeatother
\[2O(m_{\ell}) + 2O(\frac{m_{\ell}}{2}) + 2O(\frac{m_{\ell}}{4}) + \ldots + 2O(1) = O(m_{\ell}).\]

\begin{figure}[htp]
\begin{center}
\begin{overpic}[width=0.5\textwidth,tics=5]
{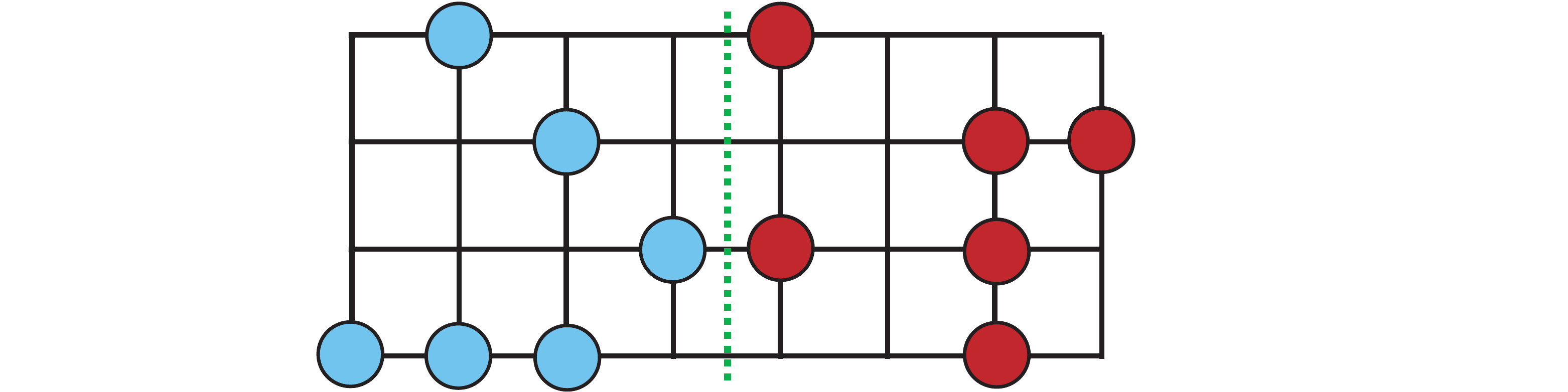}
\put(13, 11){{\small $G_1$}}
\put(75, 11){{\small $G_2$}}
\end{overpic}
\end{center}
\caption{\label{fig:sketch} Illustration of a single iteration of \sag on 
an $8 \times 4$ grid. Note that the grid is fully populated with robots and 
some are not shown in the figure. The overall grid is split in the middle by 
the dotted line to give two $4 \times 4$ grids, $G_1$ and $G_2$. The robots 
shown on $G_1$ (resp., $G_2$) have goal locations on $G_2$ (resp., $G_1$). 
In the group operation, these robots must move across the split line after 
the operation is complete. Other robots (not shown) on the grid must be where 
they were after the operation. In the next iteration, the same procedure is 
applied to $G_1$ and $G_2$ in parallel.}  
\end{figure}

The divide-and-conquer approach that we use share similarities with other 
decoupling techniques in that it seeks to break down the overall problem into 
independent sub-problems. On the other hand, it significant differs from 
previous decoupling schemes in that the decoupling in our case is {\em 
global}. Therefore, we denote the scheme as the {\em global decoupling} 
technique. 

We proceed to describe an iteration of the \sag algorithm in detail, 
which depends on following sub-routines, in a sequential manner (i.e., 
a later sub-routine makes use of the earlier ones):
\begin{itemize}
\item Concurrent exchange of multiple pairs robots embedded in a grid in 
a constant number of steps (Lemma~\ref{l:constant-flip}). 
\item Exchange of two groups of robots on a tree embedded in a grid in time
steps linear with respect to the diameter (i.e., length of the longest path) 
of the tree (Lemma~\ref{l:herd}, Lemma~\ref{l:line-shift}, and Theorem~\ref{t:tree-shift}). 
\item Partitioning a split problem into multiple exchange problems on trees 
and solving them concurrently. 
\end{itemize}
Each of these steps is covered in a sub-section that follows. 

\subsection{Pairwise Exchanges In A Constant Number of Steps}
To achieve  
$O(m_{\ell})$ makespan, \sag needs to enable concurrent robot movements. 
This is challenging because of our worst case assumption that there are as 
many robots as the number of vertices. This is where the grid graph 
assumption becomes critical: it enables the concurrent ``flipping'' or 
``bubbling'' of robots. Let $G = (V, E)$ be an $m_{\ell} \times m_s$ 
grid graph whose vertices are fully occupied by robots. Let $E' \subset 
E$ be a set of vertex disjoint edges of $G$. Suppose for each edge $e = 
(v_1, v_2) \in E'$, we would like to simulate the exchange of the two 
robots on $v_1$ and $v_2$ without incurring collision. Let us call this 
operation {\sc flip($E'$)}. We use {\sc flip($\cdot$)} to mean the operation
is applied to some unspecified set of edges, which is to be determined for 
the particular situation. 

\begin{lemma}\label{l:constant-flip}
Let $G = (V, E)$ be an $m_{\ell} \times m_s$ grid. Let 
$E' \subset E$ be a set of vertex disjoint edges. Then the {\sc flip($E'$)} 
operation can be completed in a constant number of steps. 
\end{lemma}
\begin{proof}A $3 \times 2$ rectangular grid may be viewed as a figure-8 
graph with $|V| = 6$ vertices. Applying Lemma~\ref{l:f8-feasible} to the 
$3 \times 2$ grid tells us that any two robots on such a graph can be 
exchanged without collision. Furthermore, all such exchanges can be 
pre-computed and performed in $O(1)$ (i.e., a constant number of) steps. 

To perform {\sc flip($E'$)} on an $m_{\ell} \times m_s$ grid $G$, we partition 
the grid into multiple disjoint $3 \times 2$ blocks. Using up to $4$ 
different such partitions, it is always possible to cover all edges 
of $G$. Therefore, the {\sc flip($E'$)} operation can be broken down 
into {\em parallel} two-robot exchanges on the $3\times 2$ blocks of 
these partitions. Because of the parallel nature of the two-robot 
exchanges, the overall {\sc flip($E'$)} operation can be completed $O(1)$ 
steps. As an example, Fig.~\ref{fig:constant-flip} illustrates how 
a {\sc flip($E'$)} operations can be carried out on a $7\times 5$ grid. 
Note that two partitions (the top two in Fig.~\ref{fig:constant-flip}) are 
sufficient to cover all edges below the second row (including the second row). 
Then, two more partitions (the bottom two in Fig.~\ref{fig:constant-flip}) can 
cover all edges above the second row. In the figure, the solid edges represent 
the edge set $E'$. After each partition starting from the top left one, 
two-robot exchanges can be performed which allow the removal of the edges 
covered by the partition, as shown in the subsequent picture. 
\begin{figure}[htp]
\begin{center}
\begin{overpic}[width=0.5\textwidth,tics=5]{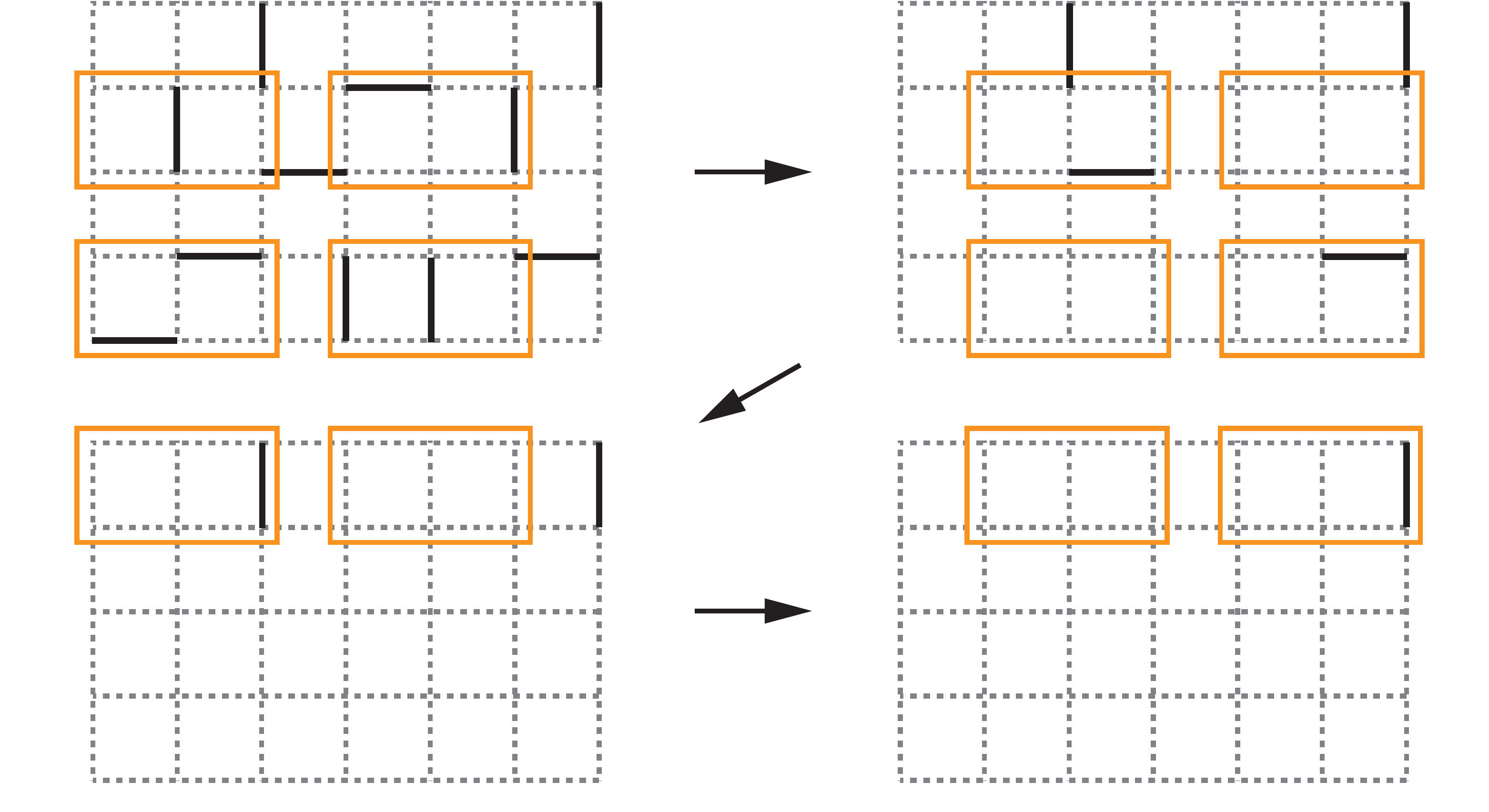}
\end{overpic}
\end{center}
\caption{\label{fig:constant-flip} Illustration of how the 
{\sc flip($E'$)} operation can be completed in a constant number of 
steps on an $m_{\ell} \times m_s$ grid, which requires up to 4 partitions of 
the grid into $3\times 2$ blocks.}  
\end{figure} ~\qed
\end{proof}
\subsection{Exchange of Groups Robots on an Embedded Tree}
Lemma~\ref{l:constant-flip}, in a nutshell, allows the concurrent exchange 
of adjacent robots to be performed in $O(1)$ steps. With 
Lemma~\ref{l:constant-flip}, to prove Theorem~\ref{t:grid-linear-makespan}, 
we are left to show that on an $m_{\ell} \times m_s$ grid, after splitting,
the group operation in the first \sag iteration can be decomposed into 
$O(m_{\ell})$ {\sc flip($\cdot$)} operations. Because each {\sc flip($\cdot$)} 
can be carried out in $O(1)$ steps, the overall makespan of the group 
operation is $O(m_{\ell})$. To obtain the desired decomposition, we need 
to maximize parallelism along the split line. We achieve the desired 
parallelism by partitioning the grid into trees with limited overlap. Each 
such tree has a limited diameter and crosses the split line. The group 
operation will then be carried out on these trees. Before detailing the 
tree-partitioning step, we show that grouping robots on trees can be done 
efficiently. We start by showing that we can effectively ``herd'' a group 
of robots to the end of a path.\footnote{We emphasize that the group 
operation and groups of robots are related but bear different meanings.}
Note that we do not require a robot in the group to go to a specific goal 
vertex; we do not distinguish robots within the group. 

\begin{lemma}\label{l:herd}
Let $P$ be a path of length $\ell$ embedded in a grid. An arbitrary group of 
up to $\lfloor \ell/2 \rfloor$ robots on $P$ can be relocated to one end of 
$P$ in $O(\ell)$ steps. Furthermore, the relocation may be performed using 
{\sc flip($\cdot$)} on $P$. 
\end{lemma}
\begin{proof}Because we are to do the relocation using parallel two-robot 
exchanges on disjoint edges based on the {\sc flip($\cdot$)} operation, 
without loss of generality, we may assume that the path is straight and 
we are to move the robots to the right end of the path. An example 
illustrating the scenario is given in Fig.~\ref{fig:herd}. For a robot 
in the group, let its initial location on the path be of distance $k$ from 
the right end. We inductively prove the claim that it takes $O(k)$ 
steps from the beginning of all moves to ``shift'' such a robot to its 
desired goal location. 
\begin{figure}[htp]
\begin{center}
    \includegraphics[width=0.5\textwidth]{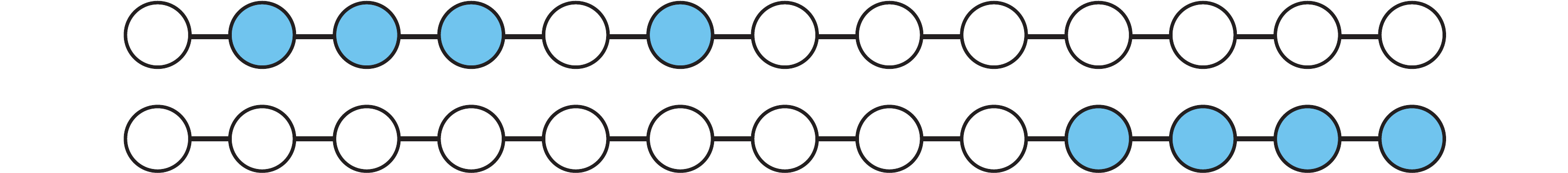} 
\end{center}
\caption{\label{fig:herd} The initial and goal configurations of 
a group of $4$ robots on a path, before and after a {\em herding} 
operation.}
\end{figure}

At the beginning (i.e., $t = 0$), let the robot on $P$ that is of 
distance $k$ to the right end be denoted as $r_k$. The hypothesis trivially 
holds for $k = 0$. Suppose it holds for $k-1$ and we need to show that 
the claim extends to $k$. If $r_k$ does not belong to the group of robots 
to be moved, then there is nothing to do. Otherwise, there are two cases. 

In the first case, robot $r_{k-1}$ does not belong to the group of robots to 
be moved. Then at $t=0$, $r_k$ and $r_{k-1}$ may be exchanged in $O(1)$ 
steps. Now $r_k$ is of distance $k - 1$ to the right and the inductive 
hypothesis yields that the rest of the moves for $r_k$ can be completed in 
$O(k - 1)$ steps. The total number of steps is then $O(k)$.

In the second case, robot $r_{k-1}$ also belongs to the group of robots to 
be moved. By the inductive hypothesis, $r_{k-1}$ can be moved to its desired
goal in $O(k-1)$ steps. However, once $r_{k-1}$ is moved to the right, 
it will allow $r_k$ to follow it with a gap between them of at most $2$. 
Once $r_{k-1}$ reaches its goal, $r_k$, whose goal is on the right of 
$r_{k-1}$, can reach its goal in $O(1)$ additional steps. The total number
of steps from the beginning is again $O(k)$. 

It is clear that all operations can be performed using {\sc flip($\cdot$)} 
on edges of $P$ when embedded in a grid. ~\qed
\end{proof}

Using the herding operation, the locations of two disjoint groups of robots,
equal in number, can also be exchanged efficiently. 

\begin{lemma}\label{l:line-shift}
Let $P$ be a path of length $\ell$ embedded in a grid. Let two groups, equal 
in number, reside on two segments of $P$ that do not intersect. Then positions 
of the two groups of robots may be exchanged in $O(\ell)$ steps without 
net movements of other robots. The relocation may be performed 
using {\sc flip($\cdot$)} on $P$. 
\end{lemma}
\begin{proof}We may again assume that $P$ is straight. An implicit assumption 
is that each group contains at most $\lfloor \ell/2 \rfloor$ robots.
Fig.~\ref{fig:line-move} illustrates an example in which two groups of $4$ 
robots each need to switch locations on such a path. 
\begin{figure}[htp]
\begin{center}
    \includegraphics[width=0.5\textwidth]{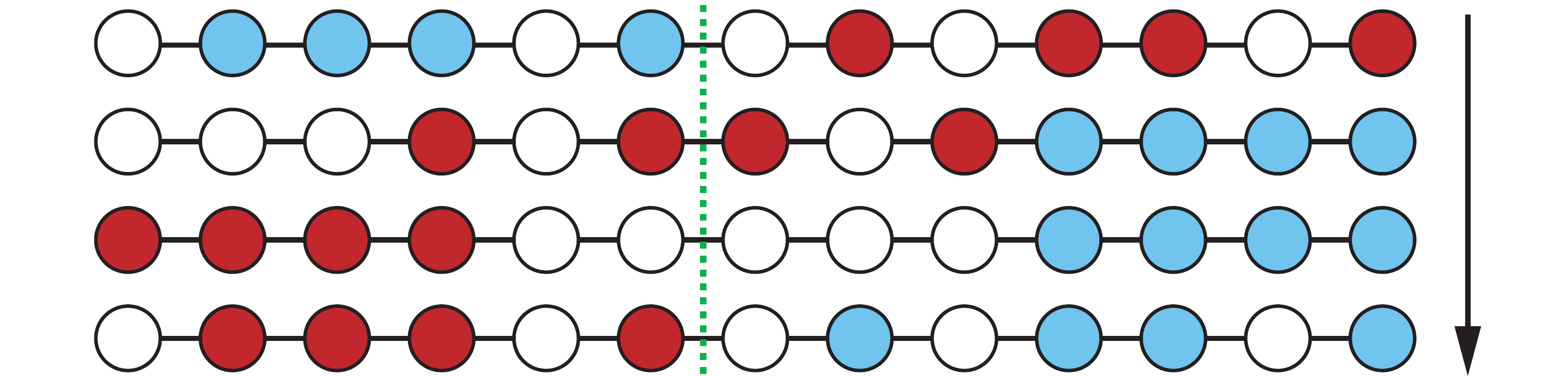} 
\end{center}
\caption{\label{fig:line-move} The initial (first row), goal (last row) 
and intermediate configurations of two groups of $4$ robots to be 
exchanged. Each group is marked with a different color/shade. The 
unshaded discs do not belong to either of the two group.}
\end{figure}

To do the grouping, we first apply a herding operation that moves one group 
of robots to one end of $P$. In Fig.~\ref{fig:line-move}, this is done to 
the group of lightly-shaded robots to move them to the right side (the 
second row of Fig.~\ref{fig:line-move}). Then, another herding operation is 
performed to move the other group to the other end of $P$ (the third row 
of Fig.~\ref{fig:line-move}). In the third and last step, two parallel 
``reversed'' herding operations are carried out on two disjoint segments 
of $P$ to move them to their desired goal locations. This is best understood
by viewing the process as applying the herding operation to the goal 
configuration. As an example, in Fig.~\ref{fig:line-move}, from the goal
configuration (last row), we may readily apply two herding operations to move 
two groups of robots to the two ends of $P$ as shown in the third row of 
the figure. Because each herding operation takes $O(\ell)$ steps, the 
overall operation takes $O(\ell)$ steps as well. It is clear that 
in the end, a robot not in the two groups will not have any net movement 
on $P$ because the relative orders of these robots (unshaded ones in 
Fig.~\ref{fig:line-move}) never change. ~\qed
\end{proof}

Next, we generalize Lemma~\ref{l:line-shift} to a tree embedded in a grid. 
On a tree graph $T$, we call a subgraph a {\em path branch} of $T$ if the 
subgraph is a path with no other attached branches. That is, all vertices 
of the subgraph have degrees one or two in $T$. 

\begin{theorem}\label{t:tree-shift}
Let $T$ be a tree of diameter $d$ embedded in a grid. Let $P$ be a length 
$\ell$ path branch of $T$. Then, a group of robots on $P$ can be exchanged 
with robots on $T$ outside $P$ in $O(d)$ steps without net movement of other 
robots. The relocation may be performed using {\sc flip($\cdot$)} on $T$. 
\end{theorem}
\begin{proof}
We temporarily limit the tree $T$ such that, after picking a proper {\em main 
path} that contains $P$ and deleting this main path, there are only paths left. 
That is, we assume all vertices with degree three or four are on a single 
path containing $P$. An example of such a tree $T$ and the exchange problem 
is given in the top row of Fig.~\ref{fig:tree-shift}. In the figure, the main
path is the long horizontal path and $P$ is the path on the left of the dotted 
split line. We call other paths off the main path {\em side branches}. 
Once this version is proven, the general version readily follows because all 
possible tree structures are considered in this special example, i.e., there 
may be either one or two branches coming out of a node on the main branch. The rest 
of the paper will only use the less general version. For ease 
of reference, for the two groups of robots, we denote the group fully on $P$ 
as $g^1$ and the other group as $g^2$. In the example, $g^1$ has a light 
shade and $g^2$ has a darker shade.
\begin{figure}[htp]
\begin{center}
\begin{overpic}[width=0.5\textwidth,tics=5]{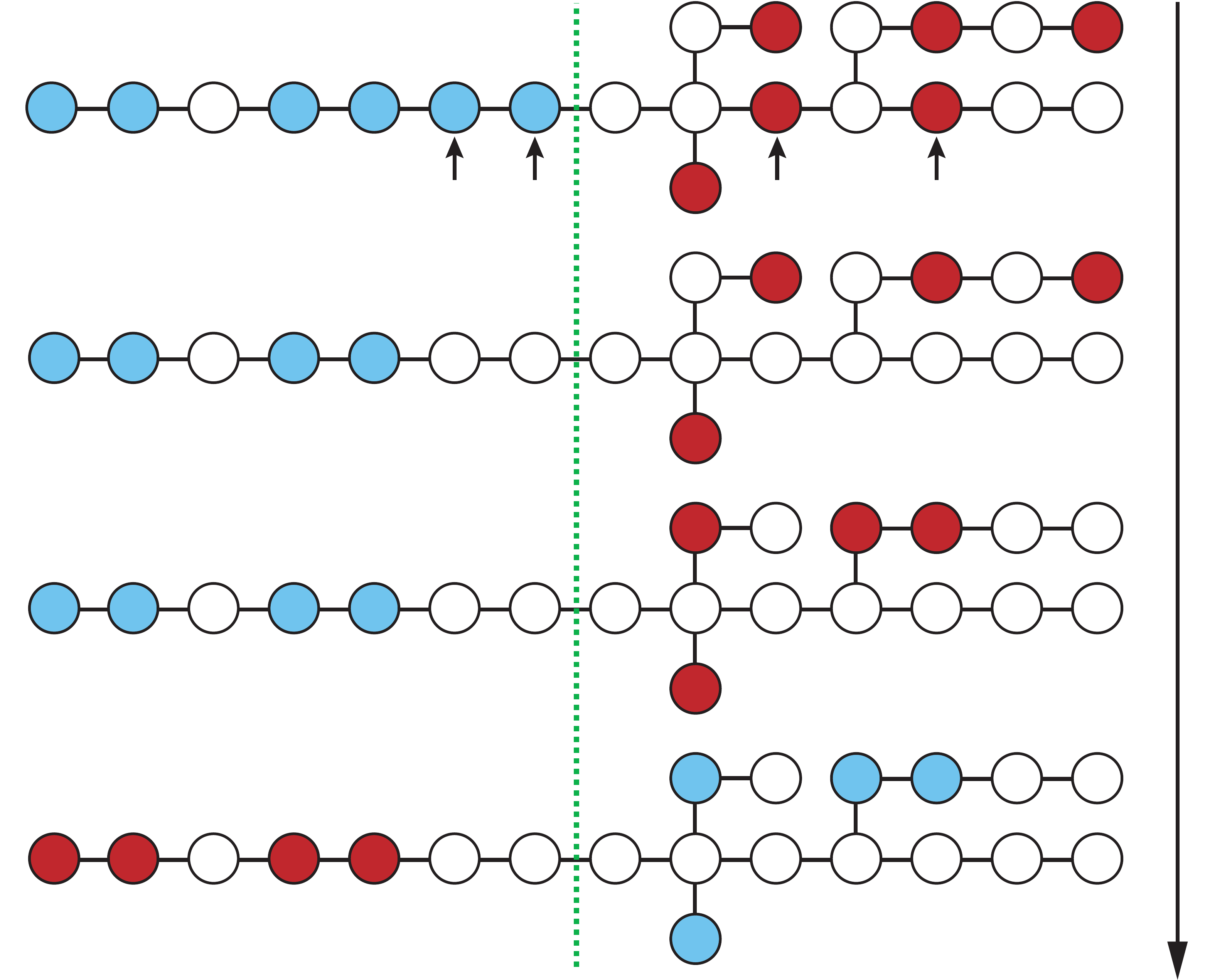}
%\put(23, 35){{\small $r^1_4$}}
%\put(29.3, 35){{\small $r^1_3$}}
%\put(35.9, 35){{\small $r^1_2$}}
%\put(42.5, 35){{\small $r^1_1$}}
\put(51, 23) {{\small $r^2_1$}}
\put(51, 36) {{\small $r^2_2$}}
\put(69.5, 42) {{\small $r^2_3$}}
\put(76.3, 42) {{\small $r^2_4$}}
\end{overpic}
\end{center}
\caption{\label{fig:tree-shift} The initial (first row) and three 
intermediate configurations in solving the problem of switching the 
location of these robots on a tree.}
\end{figure}
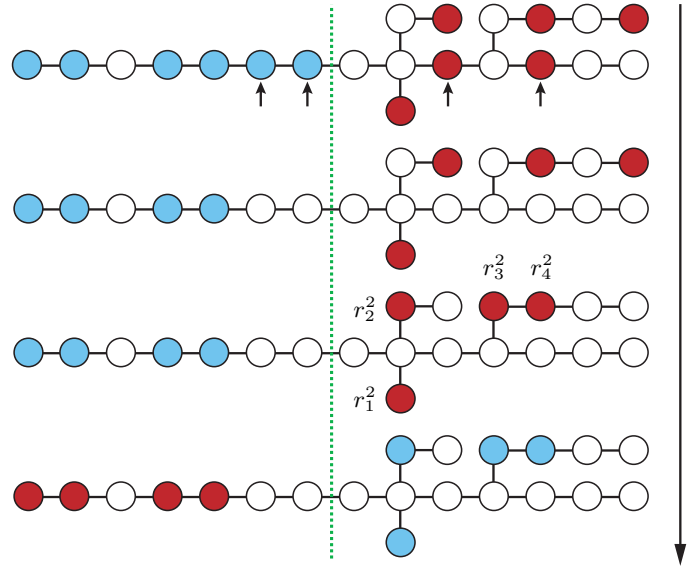

To start, we first solve part of the relocation problem on the main path, 
which can be done in $O(d)$ steps by Lemma~\ref{l:line-shift}. After 
the step, the robots involved in the first step are no longer relevant. In 
the example, this is to exchange the robots marked with small arrows in the 
first row of Fig.~\ref{fig:tree-shift}. After the relocation of these robots 
is completed, we remove their shades.

In the second step, the relevant robots in $g^2$ on the side branch are 
moved so that they are just off the main path. We also assign priorities to 
these robots based on their closeness to $P$ and break ties randomly. For 
a robot labeled $i$ in a group $g^j$, we denote the robot as $r_i^j$. For 
our example, this current step yields the third row of 
Fig.~\ref{fig:tree-shift} with the priorities marked. Since the moves 
are done in parallel and each branch is of length at most $d$, only 
$O(d)$ steps are needed. 

In the third step, robots from $g^2$ will move out of the side branches 
in the order given, one immediately after the other (when possible). For 
the example (third row of Fig.~\ref{fig:tree-shift}), $r_1^2$ will move 
first. $r_2^2$ will follow. Then $r_3^2$, followed by $r_4^2$. Using the 
same inductive argument from the proof of Lemma~\ref{l:herd}, we observe 
that all robots from $g^2$ on the side branch can be moved off the side 
branches (and reach their goals on the main path) in $O(d)$ time. As the 
relevant robots from $g^1$ also move across the split line, they will 
fill in side branches in opposite order to when the robots from $g_2$ 
are moved out of the branches. In the example, this means that the 
branch where $r_3^2$ and $r_4^2$ were on will be populated with robots from
$g^1$ first, followed by the branch where $r_2^2$ was, and finally the 
branch where $r_1^2$ was. This ensures that at the end of this step, any robot
not in $g^1$ and $g^2$ will have no net movement. The number of steps
for this is again $O(d)$. 

In the last step, we simply reverse the second step, which takes another 
$O(d)$ steps. Putting everything together, $O(d)$ steps 
are sufficient for completing the task. 

Combining all steps, only $O(d)$ steps are required to complete
the desired exchange. To see that the same conclusion holds for more general
trees with side branches that are not simple paths, we simply need to do the 
second step and third step more carefully. But, because we are only moving 
at most $O(d)$ robots, using an amortization argument, it is straightforward
to see that the $O(d)$ bound does not change. ~\qed
\end{proof}

We note that many of the operations used to prove Lemma~\ref{l:herd}, 
Lemma~\ref{l:line-shift}, and Theorem~\ref{t:tree-shift} can be combined 
without changing the outcome. However, doing so will make the proofs less 
modular. Given the focus of the current paper which is to construct a 
polynomial time algorithm with constant factor optimality guarantee, we opt 
for clarity instead of pursuing a smaller asymptotic constant. 

\subsection{Tree Forming and Robot Routing}
We proceed to prove Theorem~\ref{t:grid-linear-makespan} by showing in detail 
how to carry out a single iteration of \sag, which boils down to partitioning
the robot exchanges into robot exchanges on trees, to which 
Theorem~\ref{t:tree-shift} can then be applied. The proof itself can be 
subdivided into three steps:
\begin{itemize}
\item \textbf{Splitting and initial tree forming}, where a grid is partitioned
into two roughly equal halves and trees are initial formed across the partition
line for facilitating exchanging of two groups of robots.  
\item \textbf{Tree post-processing}, which addresses the issue where two 
initial trees might have ``$+$'' like crossovers. 
\item \textbf{Final robot routing}, which actually carries through the 
robot routing process and resolve some final issues.  
\end{itemize}

\begin{proof}[Proof of Theorem~\ref{t:grid-linear-makespan}]
\textbf{Splitting and initial tree forming}. In a split, we always split 
along the longer side of the current grid. 
Since $m_{\ell} \ge m_s$, the $m_{\ell} \times m_s$ grid is split into two 
grids of dimensions $\lceil m_{\ell}/2\rceil 
\times m_s$ and $\lfloor m_{\ell}/2\rfloor \times m_s$, respectively. 
For convenience, we denote the two split grids as $G_1$ and $G_2$, 
respectively. Recall that in the group operation, we want to exchange 
robots so that a robot with goal in $G_1$ (resp., $G_2$) resides 
in $G_1$ (resp., $G_2$) at the end of the operation. To do this 
efficiently, we need to maximize the parallelism. This is achieved 
through the computation of a set of $m_s$ trees with which we can apply 
Theorem~\ref{t:tree-shift}. We will use the example from 
Fig.~\ref{fig:split-grid} to facilitate the higher level explanation. 
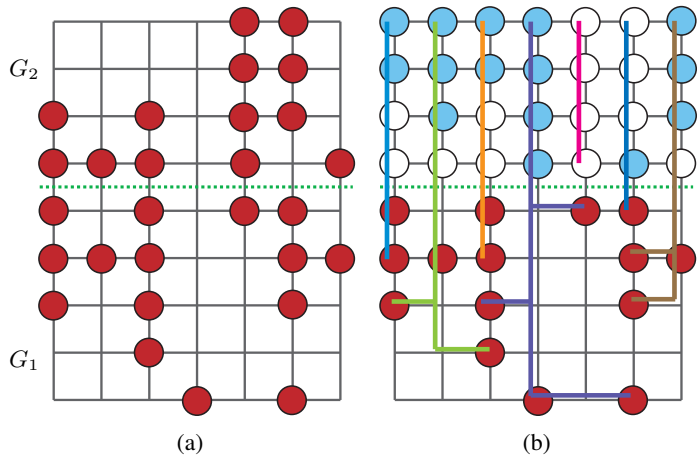
\begin{figure}[htp]
\begin{center}
\begin{overpic}[width=0.5\textwidth,tics=10]{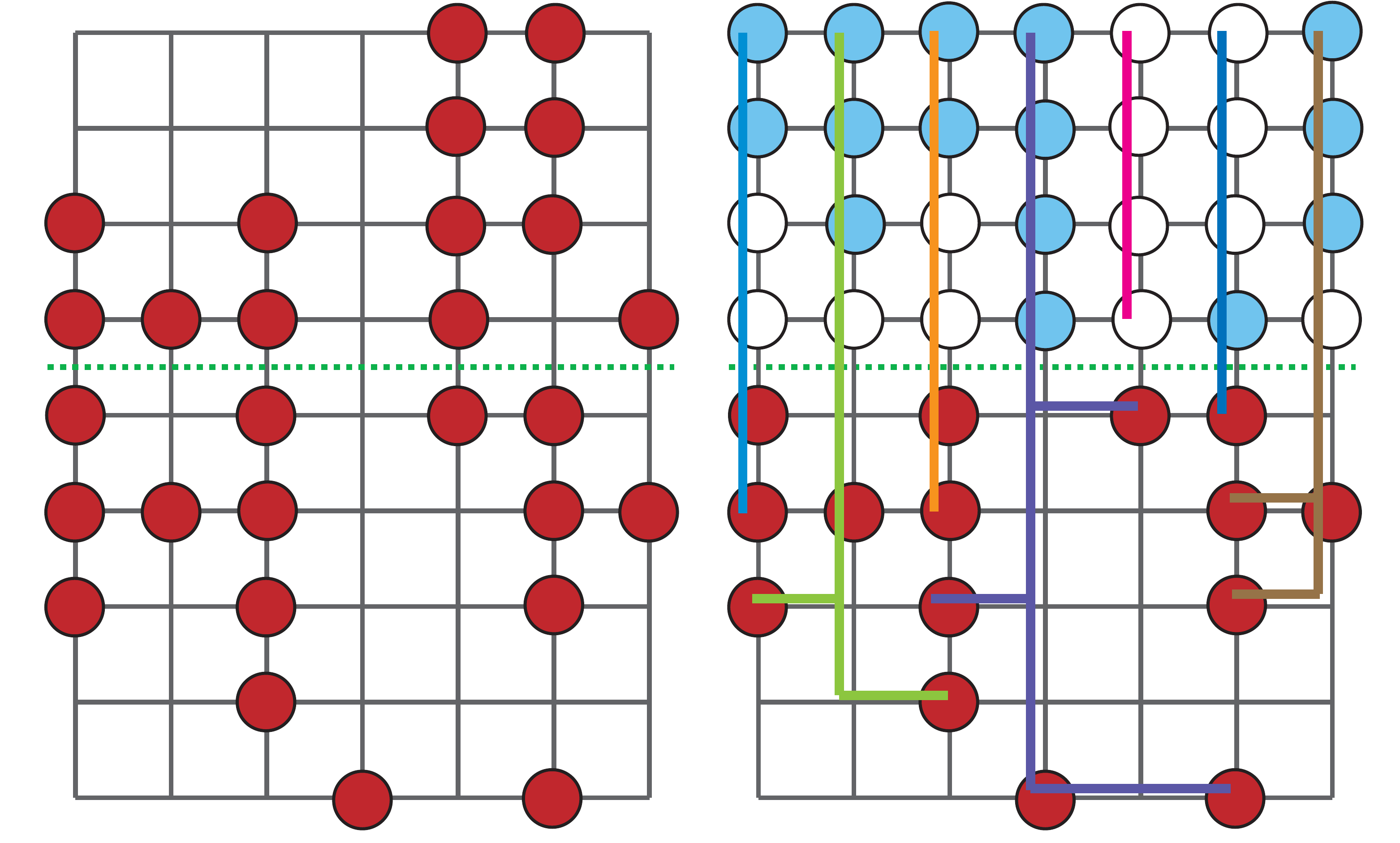} 
\put(-1, 50){{\small $G_2$}}
\put(-1, 8){{\small $G_1$}}
\put(23, -4){{\small (a)}}
\put(72.5, -4){{\small (b)}}
\end{overpic}
\end{center}
\caption{\label{fig:split-grid} (a) A $9\times 7$ grid is split into two
grids $G_1$ and $G_2$ of sizes $4 \times 7$ and $5 \times 7$, respectively. 
The dark-shaded robots' final goals are in $G_2$. (b) The grid is partitioned 
into (possibly non-disjoint) trees 
to allow the dark-shaded robots that are not already in $G_2$ to exchange 
with robots (lightly-shaded ones) that should be moved to $G_1$.} 
\end{figure}

Assume that the grid is oriented so that $m_s$ is the number of columns 
and $m_{\ell}$ is the number of rows (see Fig.~\ref{fig:split-grid}). The 
trees that will be built will be based on the columns of one of the 
split graphs, say $G_2$. A column $i$ of $G_2$ is a path of length 
$\lfloor m_{\ell}/2\rfloor - 1$ with $\lfloor m_{\ell}/2\rfloor$ robots on it. 
Suppose $k_i$ of these robots have goals outside $G_2$ (the lightly-shaded 
ones in Fig.~\ref{fig:split-grid} (b)), then it is always possible to find 
$k_i$ robots (the dark-shaded ones in Fig.~\ref{fig:split-grid}(b)) on $G_1$ 
that must go to $G_2$. A tree $T_i$ is built to allow the exchange of 
these $2k_i$ robots such that the part of $T_i$ in $G_2$ is simply column $i$.
That is, the $m_s$ trees to be built do not overlap in $G_2$. 

For a column $i$ in $G_2$ with $k_i$ robots to be moved to $G_1$, it is not 
always possible to find exactly $k_i$ robots on column $i$ of $G_1$. This
makes the construction of the trees in $G_1$ more complex. The construction 
is done in two steps. In the first step, robots to be moved to $G_2$ are grouped in a distance
optimal manner, which induces a preliminary tree structure. Focusing on 
$G_1$, we know the number of robots that must be moved across the split line 
in each column (see Fig.~\ref{fig:assign-move}). For each robot to be moved 
across the split line, the distance between the robot and all the possible 
exits of $G_1$ is readily computed. Once these distances are computed, 
a standard matching procedure can be run to assign each robot an exit point 
that minimizes the total distance traveled by these robots \cite{Kuh55, SolYu15}. 
The assignment has a powerful property that we will use later. For each robot, 
either a straight or an $L$ shaped path can be obtained based on the assignment.
Merging these paths for robots exiting from the same column then yields 
a tree for each column (see Fig.~\ref{fig:split-grid}(b)). Note that each
tree has a single vertical segment. 

\begin{figure}[htp]
\begin{center}
\begin{overpic}[width=0.5\textwidth,tics=10]{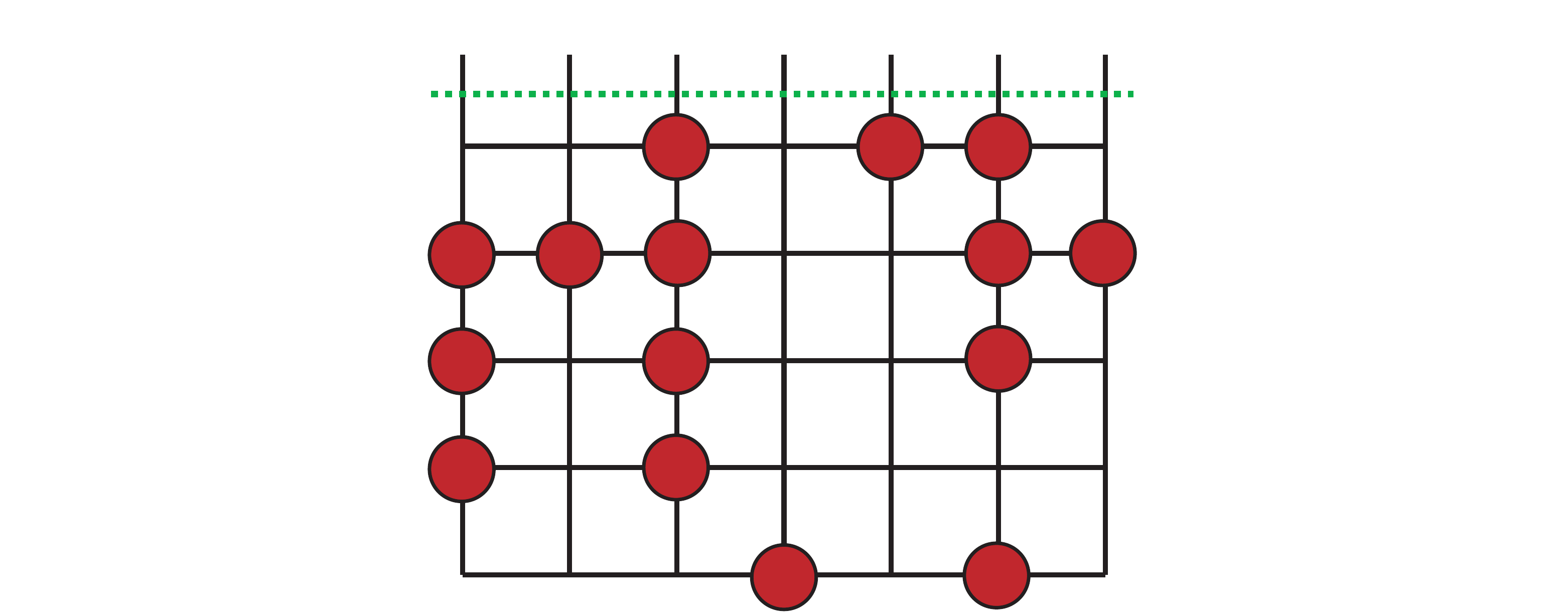}
\put(28.4, 38){{\small $2$}}
\put(35.3, 38){{\small $3$}}
\put(42, 38){{\small $2$}}
\put(48.8, 38){{\small $4$}}
\put(55.6, 38){{\small $0$}}
\put(62.6, 38){{\small $1$}}
\put(69.2, 38){{\small $3$}}
\end{overpic}
\end{center}
\caption{\label{fig:assign-move} For the example given in 
Fig.~\ref{fig:split-grid}, this figure highlights $G_1$, the robots that 
must be moved to $G_2$, and how many robot need to be moved through the 
top of $G_1$ along each column. Regarding distance, the bottom left robot
needs to travel $4$ edges to exit $G_1$ through the left most column. It 
needs to travel $4 + 7 - 1 = 10$ edges to exit from the right most 
column.}  
\end{figure}
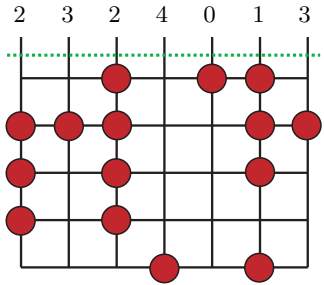

\noindent\textbf{Tree post-processing}. 
In the second step, the trees are post-processed to remove crossings 
between them. Example of such a crossing we refer to is illustrated in 
Fig.~\ref{fig:crossover-removal}(a) (dotted lines). Formally, we say 
two trees $T_1$ and $T_2$ has a {\em crossover} if a horizontal path of $T_1$
intersects with a vertical path of $T_2$, with the additional requirement 
that one of the involved horizontal path from one tree forms a $+$ with the 
vertical segment of the other tree. For example, Fig.~\ref{fig:crossover-removal}(b)
is not considered a crossover. 

\begin{figure}[htp]
\begin{center}
\begin{overpic}[width=0.5\textwidth,tics=10]{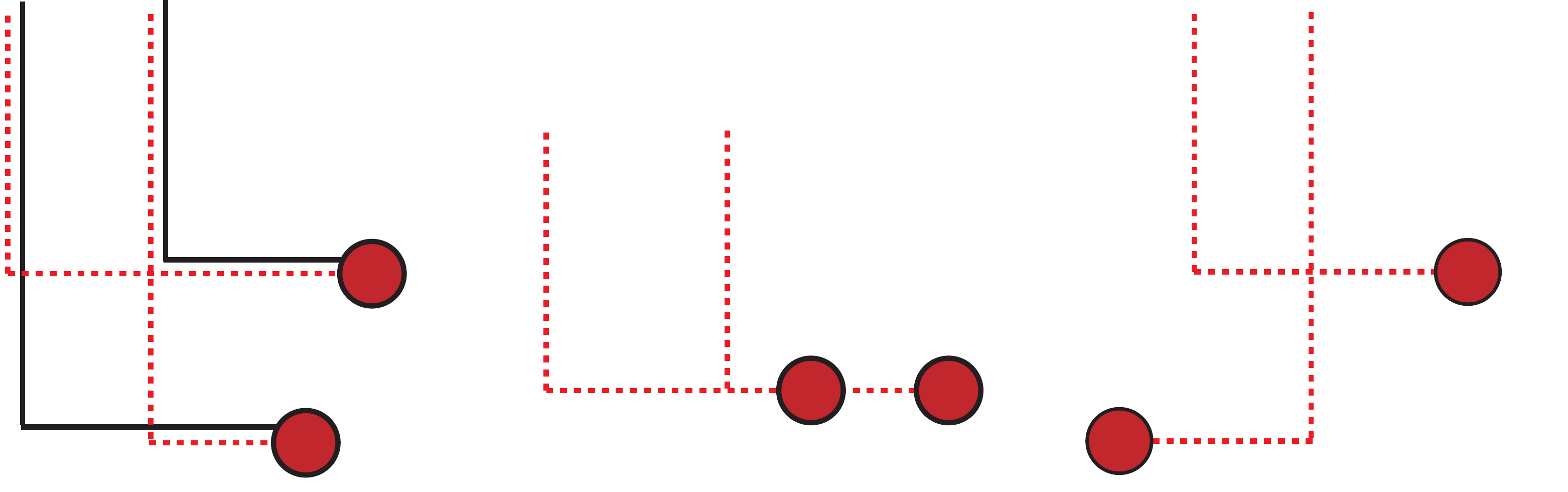}
\put(8, -4){{\small (a)}}
\put(46, -4){{\small (b)}}
\put(82, -4){{\small (c)}}
\end{overpic}
\end{center}
\caption{\label{fig:crossover-removal} (a) Example of a tree crossover 
(dotted paths) and its removal (solid paths) without increasing the total 
distance. Note that only the relevant paths of the two trees are shown. 
(b) An intersection that is not considered a crossover. (c) An impossible 
crossover scenario.}  
\end{figure}

For each crossover, we update the two trees to remove the crossover, as 
illustrated in Fig.~\ref{fig:crossover-removal}(a). The removal will not 
change the total distance traveled by the two (or more) affected robots
but will change the path for these robots. To see that the process will
end, note that one of the two involved paths is shortened. Since there 
are finite number of such paths and each path can only be shortened a 
finite number of times, the crossover removal process can get rid of 
all crossovers. We will show later this can be done in polynomial 
computation time when we perform algorithm analysis. We note here that
the crossover scenario in Fig.~\ref{fig:crossover-removal}(c) cannot 
happen because a removal would shorten the overall length, which contradicts
the assumption that these paths have the shortest total distance. 

\noindent\textbf{Final robot routing}.
At the end of the crossover removal process, we may first route all robots on 
a tree branch that do not have overlaps with other trees. However, this does
not route all robots because it is possible for the tree structures for 
different columns to overlap horizontally (see Fig.~\ref{fig:artifact}). For 
two trees that partially overlap with each other (e.g., the left and middle two 
trees in Fig.~\ref{fig:artifact}), one of the trees does not extend lower (row wise) 
than the row where the overlap occurs. Otherwise, this yields a crossover, 
which should have already been removed. For two overlapping trees
$T_1$ and $T_2$, we say $T_1$ is a {\em follower} of $T_2$ if a robot going
to $T_2$ must pass through the vertical path of $T_1$. In the example from 
Fig.~\ref{fig:artifact}, $T_1$ is a follower of $T_2$. Similarly, the 
right (green) tree is a follower of $T_1$.
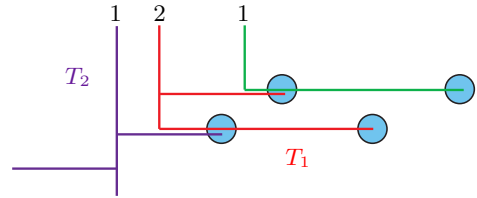
\begin{figure}[htp]
\begin{center}
\begin{overpic}[width=0.45\textwidth,tics=10]{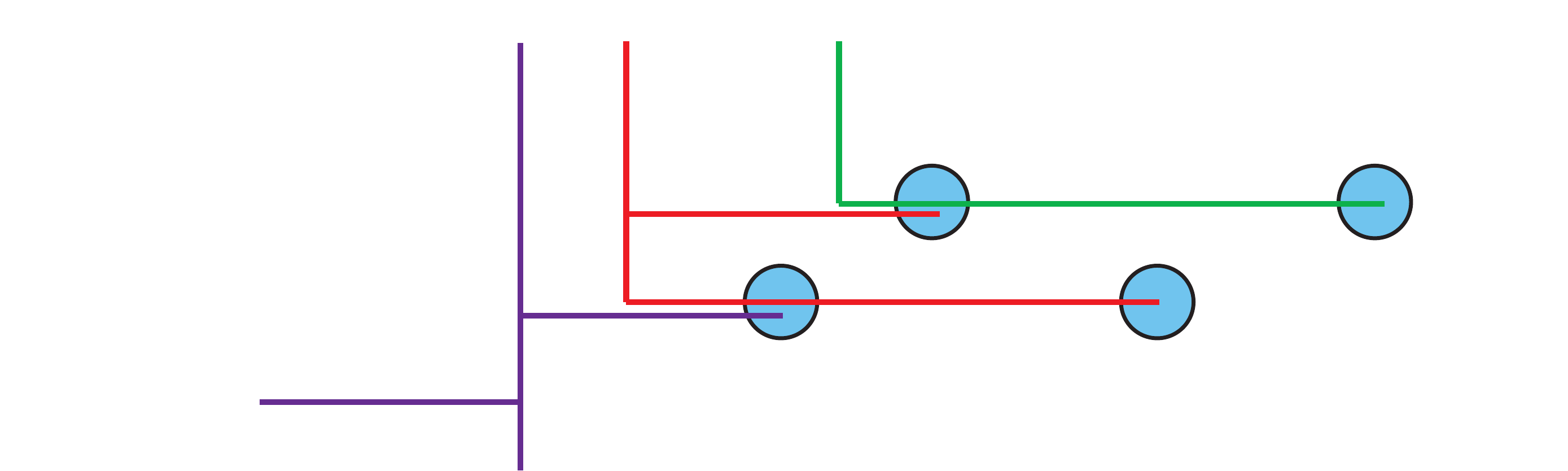}
\put(32,28){{\small $1$}}
\put(39,28){{\small $2$}}
\put(52.3,28){{\small $1$}}
\put(25,18){{\small \textcolor{purp}{$T_2$}}}
\put(60,5){{\small \textcolor{red}{$T_1$}}}
\end{overpic}
\end{center}
\caption{\label{fig:artifact} Illustration of overlaps between tree 
structures for different columns. For the four shaded robots to be moved,
one must be moved through the left most vertical path, two must be moved
through the middle vertical path, and the last must be moved through the 
right most vertical path.}  
\end{figure}

We state some readily observable properties of overlapping trees:
{\em (i)} two trees may have at most one overlapping horizontal branch
(otherwise, there must be a path crossover), {\em (ii)} because of {\em (i)}, 
any three trees cannot pair wise overlap at different rows, and {\em (iii)} 
there must be at least one tree that is not a follower, e.g., the left 
(purple) tree in Fig.~\ref{fig:artifact}. We call this tree a {\em leader}.
From a leader tree, we can recursively collect its followers, and the 
followers of these followers, and so on so forth. We call such a collection 
an {\em interacting bundle} (e.g., Fig.~\ref{fig:artifact}). 

With these properties in mind, the group operation in a \sag iteration is 
carried out as follows. Because robots to be moved from $G_2$ to $G_1$ are 
on straight vertical paths, there are no interactions among them between 
different trees. Therefore, we only need to consider interactions of robots
on $G_1$. For trees that have no overlap with other trees, 
Theorem~\ref{t:tree-shift} directly applies to complete the robot exchange 
on these trees in $O(m_{\ell})$ steps because each tree has diameter at most 
$2m_{\ell}$. In parallel, we can also complete the movement of all robots that 
should go from $G_1$ to $G_2$ which are not residing on a horizontal 
tree branch that overlaps with other trees, also in $O(m_{\ell})$ steps. 
After these robots are exchanged, we can effectively forget about them. 

After the previous step, we are left to deal with robots on overlapping 
horizontal tree branches that must be moved (e.g., the shaded robots in 
Fig.~\ref{fig:artifact}). It is clear that different interacting bundles 
do not have any interactions; we only need to focus on a single bundle.  
This is actually straightforward; we use the example from 
Fig.~\ref{fig:bundle} to facilitate the proof explanation. The routing of 
robots in this case follows a greedy approach starting from the left most 
tree, i.e., we essentially try to ``flush'' the shaded robots in the left-up
direction, which can always be realized in two phases, each of which using at 
most $O(m_{\ell})$ makespan. 
\begin{figure}[htp]
\begin{center}
\begin{overpic}[width=0.5\textwidth,tics=5]{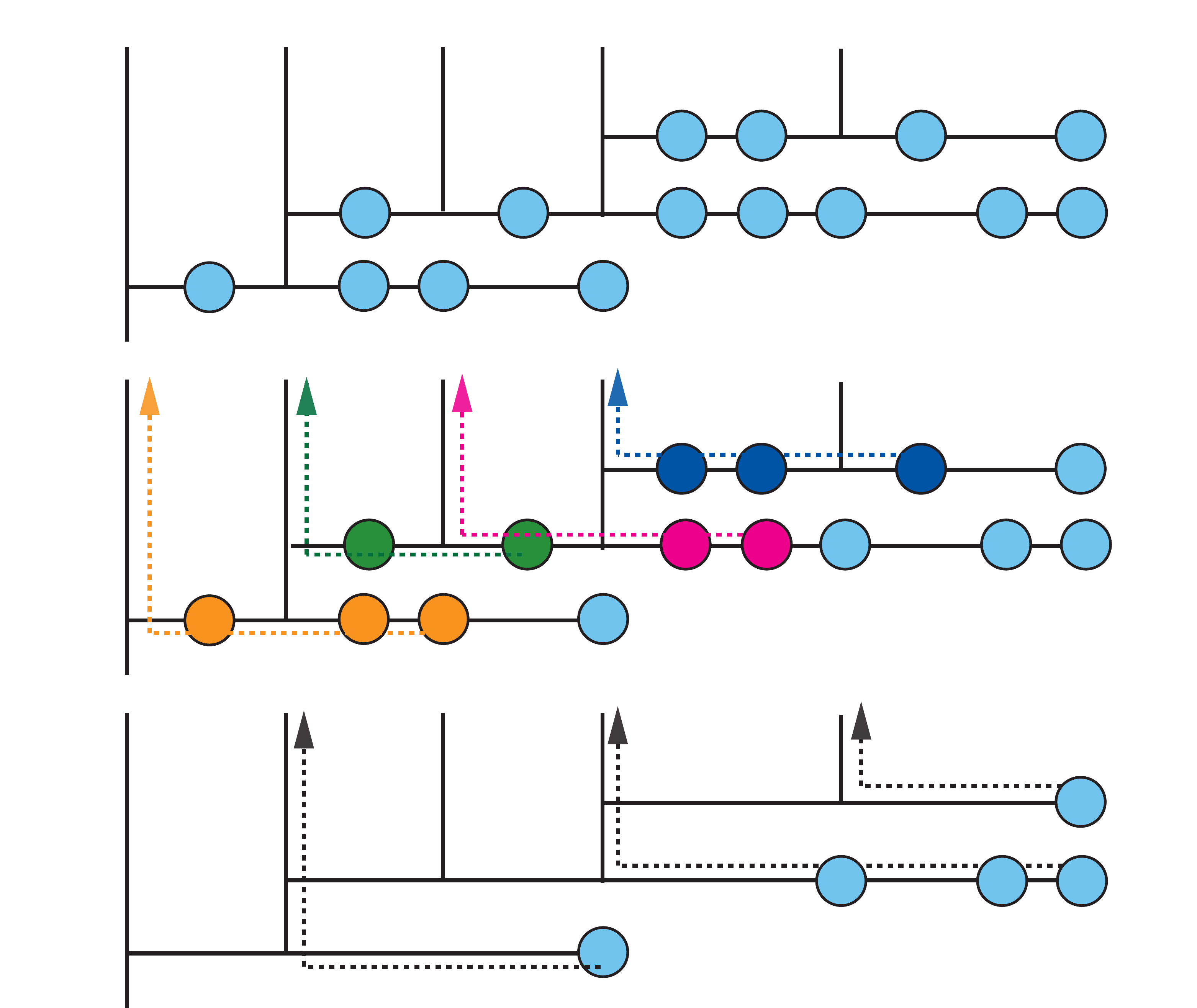}
\put(9.7,82){{\small $3$}}
\put(23,82){{\small $3$}}
\put(36,82){{\small $2$}}
\put(49.4,82){{\small $6$}}
\put(69.2,82){{\small $1$}}
\end{overpic}
\end{center}
\caption{\label{fig:bundle} An example interacting bundle in detail. The
top row is the initial configuration of the robots on the overlapping
horizontal tree branches to be moved through the vertical paths. 
The numbers on the top of the vertical paths mark how many robots should be 
moved through that path. For example, $3$ of the shaded robots must be routed 
through the left most vertical path. The middle row and the bottom rows mark 
how the exchanges can be completed in two steps or phases.}  
\end{figure}
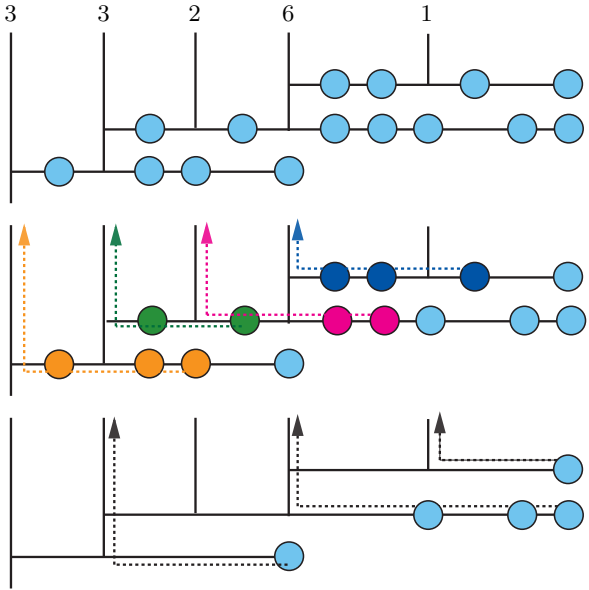

Observe that the problem can be solved for the leader tree (left most tree 
in Fig.~\ref{fig:bundle}). At the same time, for each successive follower 
tree, the movement of robots can be partially solved for these follower 
trees. The middle row of Fig.~\ref{fig:bundle} shows how this can be done 
for each tree. Formally, if a horizontal branch is shared by two trees, say 
$T_2$ and its follower $T_1$, then we obtain a simple exchange problem of 
moving a few robots through a path on $T_2$. In the figure, these are the 
first and fourth trees from the left, with the dotted lines marking the 
path. If a horizontal branch is shared by three or more trees, we get an 
exchange problem on a tree. In the figure, the middle three trees create 
such a problem. For this, Theorem~\ref{t:tree-shift} applies with minor 
modifications. All the exchange problems can be carried out in parallel 
because there is no further interaction between them. It takes $O(m_{\ell})$ 
steps to complete, after which we are left with another set of exchange 
problems, each of which is on a path (e.g., the three problems in the 
last row of Fig.~\ref{fig:bundle}). Lemma~\ref{l:line-shift} applies to 
yield $O(m_{\ell})$ required steps. 

Stitching everything together, the first iteration of \sag on an $m_{\ell} 
\times m_s$ grid can be completed in $O(m_{\ell})$ steps. In the next 
iteration, we are working with grids of sizes $\lceil m_{\ell}/2 
\rceil \times m_s$, which requires $O(\max\{\lceil m_{\ell}/2 \rceil, m_s\})$
steps. Following the simple recursion, which terminates after 
$O(\log m_{\ell})$ iterations, we readily obtain that $O(m_{\ell})$ steps are 
sufficient for solving the entire problem. ~\qed
\end{proof}

\section{Complexity and Solution Optimality Properties of the \sag Algorithm}
\label{section:complexity}

In this section, we establish two key properties of \sag, namely, its 
polynomial running-time and asymptotic solution optimality. 
\subsection{Time Complexity of \sag}
The \sag algorithm is outlined in Algorithm~\ref{alg:grid}, which summarizes
the results from Section~\ref{section:makespan-algorithm} in the form of an algorithm. At 
Lines~\ref{line:split}-\ref{line:match}, a partition of the current grid 
$G$ is made, over which initial path planning is performed to generate the
trees for grouping the robots into the proper subgraph. Then, at 
Line~\ref{line:remove-crossover}, crossovers are resolved. At 
Line~\ref{line:schedule}, the final paths are scheduled, from which the 
robot moves can be extracted. This step also yields where 
each robot will end up at in the end of the iteration, which becomes the 
initial configuration for the next iteration (if there is one). 
After the main iteration steps are complete, at 
Lines~\ref{line:recurse}-\ref{line:self}, the algorithm recursively 
calls itself on smaller problem instances. The special case here is when 
the problem is small enough (Line~\ref{line:solve-small}), in which case 
the problem is directly solved without further splitting. 
\begin{algorithm}
    \SetKwInOut{Input}{Input}
    \SetKwInOut{Output}{Output}
    \SetKwComment{Comment}{\%}{}
    \Input{$G = (V, E)$: an $m_{\ell} \times m_s$ grid graph \\ 
		$X_I$: initial robot configurations\\
		$X_G$: goal robot configurations\\}
    \Output{$M = \langle M_1, M_2, \ldots \rangle$: a sequence of {\em moves} }

\vspace{0.1in}
		
\Comment{\small Run matching and construct initial trees}
\vspace{0.025in}
$(G_1, G_2) \leftarrow \textsc{Split}(G)$ \label{line:split}\\
$\mathcal P \leftarrow \textsc{MatchAndPlanPath}(G, X_I, X_G)$ \label{line:match}
\vspace{0.075in}

\Comment{\small Remove crossovers}
\vspace{0.025in}
$\mathcal P' \leftarrow \textsc{ResolveCrossovers}(\mathcal P)$\label{line:remove-crossover}
\vspace{0.075in}

\Comment{\small Schedule the sequence of moves}
\vspace{0.025in}
$(M, X_I')\leftarrow \textsc{ScheduleMoves}(\mathcal P')$\label{line:schedule}
\vspace{0.075in}

\Comment{\small Recursively solve smaller sub-problems}
\ForEach{$G_i, i = 1, 2$}{\label{line:recurse}
\uIf{$row(G_i) \le 3$ and $col(G_i) \le 3$}
{$M = M + \textsc{Solve}(G_i, X_I'|_{G_i}, X_G|_{G_i})$}\label{line:solve-small}
\Else{$M = M + \sag(G_i, X_I'|_{G_i}, X_G|_{G_i})$}\label{line:self}
}

\vspace{0.075in}

\Return{$M$}
\caption{\sagalgo($G$, $X_I$, $x_G$)} \label{alg:grid}
\end{algorithm}

We now proceed to bound the running time of \sag. 
%Since the algorithm is 
%not meant to be implemented in its current form, our main aim here is to 
%show that it runs in polynomial time. 
It is straightforward to see that 
the {\sc Split} routine takes $O(|V|) = O(m_{\ell}m_s)$ running time. 
\textsc{MatchAndPlanPath} can be implemented using the standard Hungarian 
algorithm \cite{Kuh55}, which runs in $O(|V|^3)$ time. 

For \textsc{ResolveCrossovers}, we may implement it by starting with an 
arbitrary robot that needs to be moved across the split line and check 
whether the path it is on has crossovers that need to be resolved. Checking 
one path with another can be done in constant time because each path has only 
two straight segments. Detecting a crossover then takes up to $O(|V|)$ running 
time. We note that, as a crossover is resolved, one of the two paths will end 
up being shorter (see, e.g., Fig.~\ref{fig:crossover-removal}). We then repeat 
the process with this shorter path until no more crossover exists. Naively, 
because the path keeps getting shorter, this process will end in at most 
$O(|V|)$ steps. Therefore, all together, \textsc{ResolveCrossovers} can be 
completed in $O(|V|^3)$ time. 

The \textsc{ScheduleMoves} routine simply extracts information from the 
already planned path set $\mathcal P'$ and can be completed in $O(|V|)$ running time. 
The \textsc{Solve} routine takes constant time. 

Adding everything up, an iteration of \sag can be carried out in $O(|V|^3)$ 
time using a naive implementation. Summing over all iterations, the total 
running time is 
\[
O(|V|^3) + 2O((\frac{|V|}{2})^3) + 4O((\frac{|V|}{4})^3) + \ldots = O(|V|^3),
\]
which is low-polynomial with respect to the input size. 

\subsection{Optimality Guarantees}
Having established that \sag is a polynomial time algorithm that solves \mpp
with sub-linear makespan, we now show that \sag is an $O(1)$-approximate 
makespan optimal algorithm for \mpp in the average case. Moreover, \sag 
also computes a constant factor distance optimal solution in a weaker sense. 
To establish these, we first show that for fixed $m_{\ell}m_s$ that is large, 
the number of \mpp instances with $o(m_{\ell} + m_s)$ makespan is negligible. 

\begin{lemma}\label{l:distribution}
The fraction of \mpp instances on a fixed graph $G$ as an 
$m_{\ell}\times m_s$ grid with $o(m_{\ell}+m_s)$ makespan is no more than 
$(\frac{1}{2})^{\frac{m_{\ell}}{4}}$ for sufficiently large $m_{\ell}m_s$. 
\end{lemma}
\begin{proof}Let $(G, X_I, X_G)$ be an \mpp instance wth $G$ being a 
$m_{\ell}\times m_s$ grid. Without loss of generality, we may assume that 
$X_I$ is arbitrary and $X_G$ is a row-major ordering of the robots, i.e., 
with the $i$-th row containing the robots labeled $(i-1)m_{\ell} + 1, 
(i-1)m_{\ell} + 2, \ldots, im_{\ell}$, in that order. Then, over all 
possible instances, $\frac{1}{4}$ of instances 
have $X_I$ with robot $1$ having a distance of $\frac{1}{4}(m_{\ell} + m_s)$ 
or less from robot $1$'s location in $X_G$ (note that the number is 
$\frac{1}{4}$ instead of $\frac{1}{16}$ because $m_s$ maybe as small as 2). 
Let these $\frac{1}{4}$ instances be $\mathcal P_1$. Among $\mathcal P_1$, 
again about $\frac{1}{4}$ of instances have have $X_I$ with robot $2$ having 
a distances of $\frac{1}{4}(m_{\ell} + m_s)$ or less from robot $2$'s location in
$X_G$. Following this reasoning and limiting to the first $\frac{m_{\ell}}{4}$ 
robots, we may conclude that the fraction of instances with $o(m_{\ell} + m_s)$ 
makespan is no more than $(\frac{1}{2})^{\frac{m_{\ell}}{4}}$ for sufficiently 
large $m_{\ell}m_s$.  ~\qed
\end{proof}

Lemma~\ref{l:distribution} is conservative but sufficient for establishing 
that \sag delivers an average case $O(1)$-approximation for makespan. That 
is, since only an exponentially small number of instances have small makespan, 
the average makespan ratio is clearly constant, that is, 

\begin{theorem}\label{t:ave-const-makespan}On average and in polynomial time, 
\sag computes $O(1)$-approximate makespan optimal solutions for \mpp. 
\end{theorem}

\begin{proof}
For a fixed $G$ as an $m_{\ell} \times m_s$ grid, Lemma~\ref{l:distribution}
says that no more than a $(\frac{1}{2})^{\frac{m_{\ell}}{4}}$ fraction of 
instances have sub-linear makespan (the minimum possible makespan is $1$). For
the rest of the instances, their makespan is linear, i.e., $\Omega(m_{\ell}+m_s)$.  
On the other hand, Algorithm~\ref{alg:grid} guarantees a makespan of 
$O(m_{\ell} + m_s)$. We may then compute the expected (i.e., average) makespan 
ratio over all instances of \mpp for the same $G$ as 
no more than (for sufficiently large $m_{\ell}$)
\[
(\frac{1}{2})^{\frac{m_{\ell}}{4}}\frac{O(m_{\ell}+m_s)}{1}
+ (1-(\frac{1}{2})^{\frac{m_{\ell}}{4}})\frac{O(m_{\ell}+m_s)}{\Omega(m_{\ell}+m_s)}
= O(1).
\]
This establishes the claim of the theorem. 
~\qed
\end{proof}

\sag can also provide guarantees on total distance optimality. In this case, 
because every robot contributes to the total distance, a weaker guarantee
is ensured (in the case of makespan, one robot's makespan dominates the 
makespan of all other robots). This leads us to work with a typical (i.e., 
average) \mpp instance instead of working with averages of makespan optimality 
ratio over all instances. That is, for the makespan case, we first compute 
the optimality ratio for each instance, which is subsequently averaged. For 
total distance, we work with a typical random instance and compute the 
optimality ratio for such a typical instance. 

\begin{theorem}\label{t:exp-const-distance}For an average \mpp instance,
\sag computes an $O(1)$-approximate total distance optimal solution. 
\end{theorem}
\begin{proof}
For an average \mpp instance, each robot incurs a minimum travel 
distance of $\Omega(m_{\ell})$; therefore, the minimum total distance for 
all robots, in expectation, is $\Omega(m_sm_{\ell}^2)$ because there are 
$m_sm_{\ell}$ robots. On the other hand, because \sag produces a solution with an 
$O(m_{\ell})$ makespan, each robot travels a distance of $O(m_{\ell})$. 
Summing this over all robots, the solution from \sag has a total distance 
of $O(m_sm_{\ell}^2)$. This matches the lower bound $\Omega(m_sm_{\ell}^2)$ .  ~\qed
\end{proof}

\section{Extensions}\label{section:general}
In this section, we show that \sag readily generalizes to environments other 
than 2D rectangular grids, including high dimensional grids and continuous 
environments. 
\subsection{High Dimensions}
\sag can be extended to work for grids of arbitrary dimensions. For 
dimensions $d \ge 2$, let the grid be $m_1 \times \ldots \times m_d$. Two 
updates to \sag are needed to make it work for higher dimensions. First, the 
split line should be updated to a split plane of dimension $d - 1$. In 
the case of $d=2$, two iterations will halve all dimensions. In the case of 
general $d$, $d$ iterations are required. Thus, the approach produces a 
makespan of $O(d(m_1 + \ldots m_d))$ where $m_i$ is the side length of 
$i$-th dimension. Second, the crossover check becomes more complex; each 
check now takes $O(d)$ time instead of $O(1)$ time because each path, though 
still having up to two straight pieces, requires $O(d)$ coordinates to 
describe. Other than these changes, the rest of \sag continues to work with
some minor modifications to the scheduling procedure (which can again be 
proven to be correct using inductive proofs). The updated \sag algorithm for 
dimension $d$ therefore runs in $O(d^2|V|^3)$ time because $d$ iterations of 
split and group are need to halve all dimensions and each iteration takes 
$O(d|V|^3)$ time. The optimality guarantees, e.g., Theorems~\ref{t:ave-const-makespan}
and~\ref{t:exp-const-distance}, also carry over. 

\subsection{Well-Connected Environments}
The selection of $G$ as a grid plays a critical role in proving the 
desirable properties of \sag. In particular, two features of grid graphs 
are used. First, grids are composed of small cycles, which allow the 2-switch 
operation to be carried out locally. This in turn allows multiple 2-switch 
operations to be carried out in parallel. Second, restricting to two adjacent 
rows (or columns) of a rectangular grid (e.g., row 4 and row 5 in 
Fig.~\ref{fig:split-grid}(a)), multiple 2-switches can be completed between
these two rows in a constant number of steps. As long as the 
environment possesses these two features, \sag works. We call such 
environments {\em well-connected}.  

More precisely, a well-connected environment, $\mathcal E$, is one with 
the following properties. Let $G$ be an $m_{\ell}\times m_s$ rectangular 
grid that contains $\mathcal E$. Unlike earlier grids, here, $G$ is not 
required to have unit edge lengths; a cell of $G$ is only required to be 
of rectangular shape with $O(1)$ side lengths. Let $r_1$ and $r_2$ be two 
arbitrary adjacent rows of $G$, and let $c_1 \in r_1$, $c_2 \in r_2$ be 
two neighboring cells (see, e.g., Fig.~\ref{fig:grid-like}). The only 
requirement over $\mathcal E$ is that a robot in $c_1$ and a robot in 
$c_2$ may exchange residing cells locally, without affecting the 
configuration of other robots. In terms of the example in 
Fig.~\ref{fig:grid-like}, the two shaded robots (other robots are not drawn) 
must be able to exchange locations in constant makespan within a region of 
constant radius. The requirement then implies that parallel exchanges of 
robots between $r_1$ and $r_2$ can be performed with a constant makespan. 
The same requirement applies to two adjacent columns of $G$. 
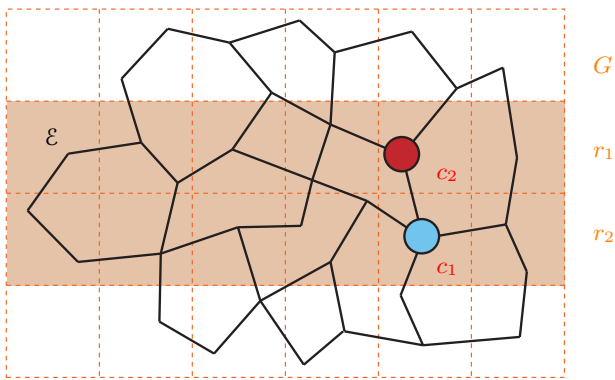
\begin{figure}[htp]
\begin{center}
\begin{overpic}[width=0.4\textwidth,tics=10]{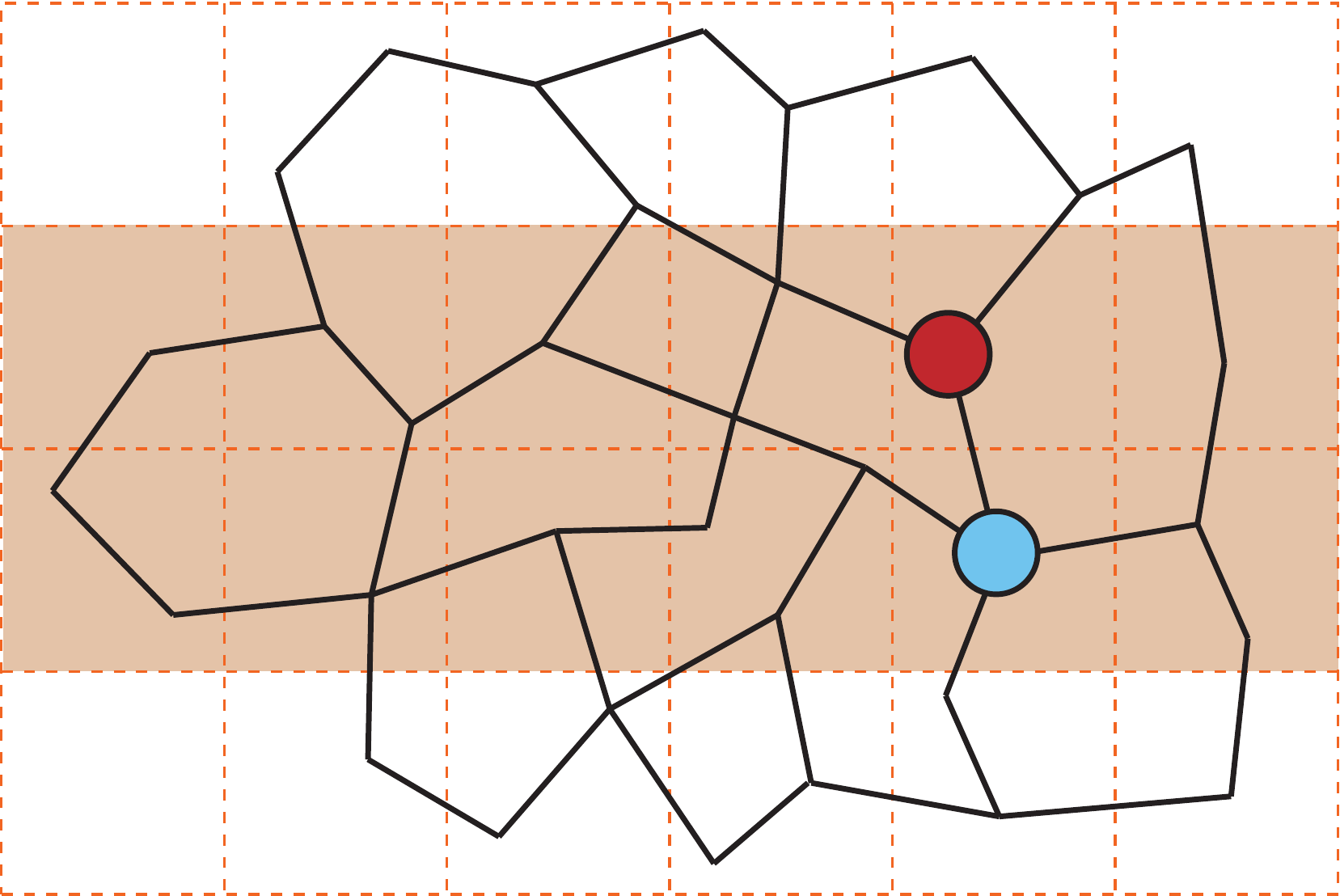}
\put(105,55){{\small \textcolor{orange}{$G$}}}
\put(7,42){{\small $\mathcal E$}}
\put(105,40){{\small \textcolor{orange}{$r_1$}}}
\put(105,25){{\small \textcolor{orange}{$r_2$}}}
\put(77,19){{\small \textcolor{red}{$c_1$}}}
\put(77,36){{\small \textcolor{red}{$c_2$}}}
\end{overpic}
\end{center}
\caption{\label{fig:grid-like}Illustration of a well-connected non-grid graph.}  
\end{figure}
Subsequently, given an arbitrary well-connected environment $\mathcal E$
and an initial robot configuration $X_I$, the steps from \sag can be readily 
applied to reach an arbitrary $X_G$ that is a permutation of $X_I$. As 
long as pairwise robot exchanges can be computed efficiently, the overall 
generalized \sag algorithm also runs efficiently while maintaining the 
optimality guarantees. We note 
%(details are omitted due to limited space) 
that the definition of well-connectedness can be further generalized to 
certain continuous settings. Fig.~\ref{fig:grid-like-ex} provides a discrete 
example and a continuous example of well-connected settings, which include 
both the environment and the robots. 

\begin{figure}[htp]
\begin{center}
\begin{overpic}[width=0.45\textwidth,tics=10]{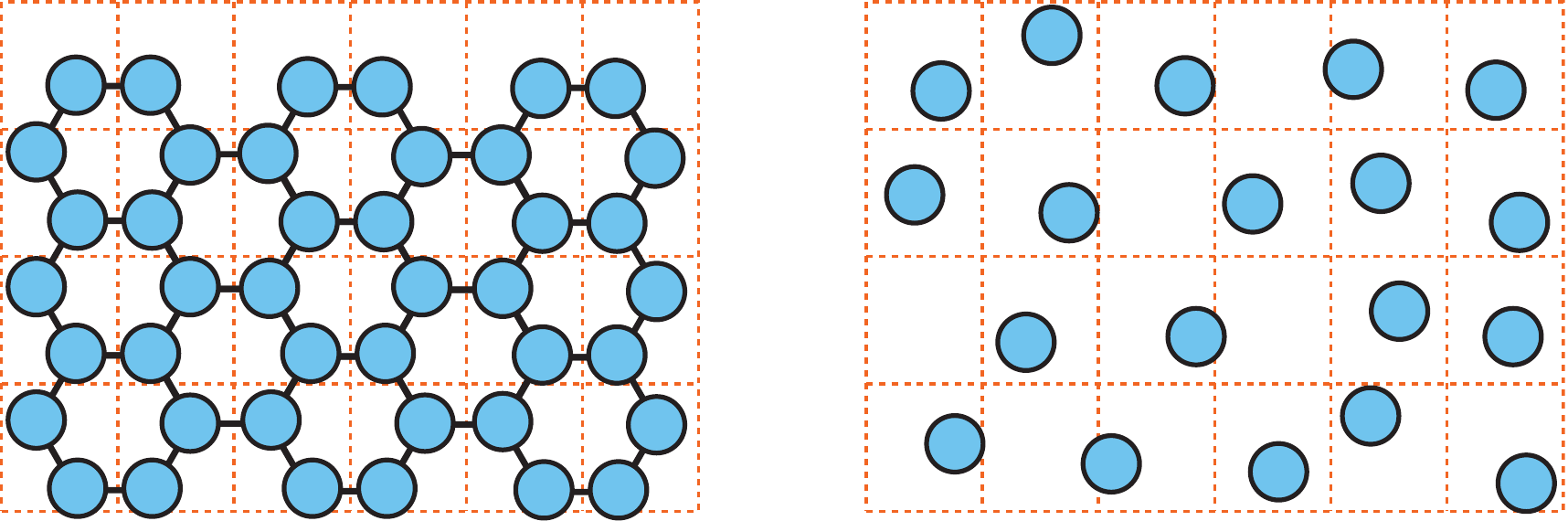}
\end{overpic}
\end{center}
\caption{\label{fig:grid-like-ex}Examples of well-connected settings, with
both environments and robots.}  
\end{figure}

As mentioned in Sec.~\ref{section:introduction}, well-connected environments are 
frequently found in real-world applications, e.g., automated warehouses at 
Amazon and road networks in cities like Manhattan. Our theoretical results 
imply that such environments are in fact quite optimal in their design 
in terms of being able to efficiently route robots. 

\section{Conclusion and Future Work}\label{section:conclusion}
In this work, we developed a low-polynomial time algorithm, \sag, for 
solving the multi-robot path planning problem in grids and grid-like, 
well-connected environments. The solution produced by 
\sag is within a constant factor of the best possible makespan on 
average. In a weaker sense, \sag also provide a constant factor 
approximation on total distance optimality. \sag applies to problems 
with the maximum possible density in graph-based settings and supports 
certain continuous problems as well.

The development of \sag opens up many possibilities for promising 
future work. On the theoretical side, \sag gets us closer to the goal
of a finding a PTAS (polynomial time approximation scheme) for optimal
multi-robot path planning. Also, it would be desirable to remove the 
probabilistic element (i.e., the ``in expectation'' part) from the 
guarantees. On the practical side, noting that we have only looked at
the case with the highest robot density, it is
promising to exploit the combination of global decoupling and network 
flow techniques to seek more optimal algorithms for cases with 
lower robot density.

\bibliographystyle{format/IEEEtran}
\bibliography{jingjin}

\begin{thebibliography}{10}
\providecommand{\url}[1]{#1}
\csname url@rmstyle\endcsname
\providecommand{\newblock}{\relax}
\providecommand{\bibinfo}[2]{#2}
\providecommand\BIBentrySTDinterwordspacing{\spaceskip=0pt\relax}
\providecommand\BIBentryALTinterwordstretchfactor{4}
\providecommand\BIBentryALTinterwordspacing{\spaceskip=\fontdimen2\font plus
\BIBentryALTinterwordstretchfactor\fontdimen3\font minus
  \fontdimen4\font\relax}
\providecommand\BIBforeignlanguage[2]{{%
\expandafter\ifx\csname l@#1\endcsname\relax
\typeout{** WARNING: IEEEtran.bst: No hyphenation pattern has been}%
\typeout{** loaded for the language `#1'. Using the pattern for}%
\typeout{** the default language instead.}%
\else
\language=\csname l@#1\endcsname
\fi
#2}}

\bibitem{goldreich2011finding}
O.~Goldreich, ``Finding the shortest move-sequence in the graph-generalized
  15-puzzle is np-hard,'' in \emph{Studies in Complexity and Cryptography.
  Miscellanea on the Interplay between Randomness and Computation}.\hskip 1em
  plus 0.5em minus 0.4em\relax Springer, 2011, pp. 1--5.

\bibitem{Yu2015IntractabilityPlanar}
J.~Yu, ``Intractability of optimal multi-robot path planning on planar
  graphs,'' \emph{IEEE Robotics and Automation Letters}, vol.~1, no.~1, pp.
  33--40, 2016.

\bibitem{banfi2017intractability}
J.~Banfi, N.~Basilico, and F.~Amigoni, ``Intractability of time-optimal
  multirobot path planning on 2d grid graphs with holes,'' \emph{IEEE Robotics
  and Automation Letters}, vol.~2, no.~4, pp. 1941--1947, 2017.

\bibitem{HanRodYu18IROS}
S.~D. Han, E.~J. Rodriguez, and J.~Yu, ``Sear: A polynomial-time expected
  constant-factor optimal algorithmic framework for multi-robot path
  planning,'' in \emph{Proceedings IEEE/RSJ International Conference on
  Intelligent Robots \& Systems}, 2018.

\bibitem{yu2017expected}
J.~Yu, ``Expected constant-factor optimal multi-robot path planning in
  well-connected environments,'' in \emph{Multi-Robot and Multi-Agent Systems
  (MRS), 2017 International Symposium on}.\hskip 1em plus 0.5em minus
  0.4em\relax IEEE, 2017, pp. 48--55.

\bibitem{ErdLoz86}
M.~A. Erdmann and T.~Lozano-P\'erez, ``On multiple moving objects,'' in
  \emph{Proceedings IEEE International Conference on Robotics \& Automation},
  1986, pp. 1419--1424.

\bibitem{LavHut98b}
S.~M. LaValle and S.~A. Hutchinson, ``Optimal motion planning for multiple
  robots having independent goals,'' \emph{IEEE Transactions on Robotics \&
  Automation}, vol.~14, no.~6, pp. 912--925, Dec. 1998.

\bibitem{GuoPar02}
Y.~Guo and L.~E. Parker, ``A distributed and optimal motion planning approach
  for multiple mobile robots,'' in \emph{Proceedings IEEE International
  Conference on Robotics \& Automation}, 2002, pp. 2612--2619.

\bibitem{Sil05}
D.~Silver, ``Cooperative pathfinding,'' in \emph{The 1st Conference on
  Artificial Intelligence and Interactive Digital Entertainment}, 2005, pp.
  23--28.

\bibitem{Rya08}
M.~R.~K. Ryan, ``Exploiting subgraph structure in multi-robot path planning,''
  \emph{Journal of Artificial Intelligence Research}, vol.~31, pp. 497--542,
  2008.

\bibitem{JanStu08}
R.~Jansen and N.~Sturtevant, ``A new approach to cooperative pathfinding,'' in
  \emph{In International Conference on Autonomous Agents and Multiagent
  Systems}, 2008, pp. 1401--1404.

\bibitem{LunBer11}
R.~Luna and K.~E. Bekris, ``Push and swap: Fast cooperative path-finding with
  completeness guarantees,'' in \emph{Proceedings International Joint
  Conference on Artificial Intelligence}, 2011, pp. 294--300.

\bibitem{StaKor11}
T.~Standley and R.~Korf, ``Complete algorithms for cooperative pathfinding
  problems,'' in \emph{Proceedings International Joint Conference on Artificial
  Intelligence}, 2011, pp. 668--673.

\bibitem{BerSnoLinMan09}
J.~van~den Berg, J.~Snoeyink, M.~Lin, and D.~Manocha, ``Centralized path
  planning for multiple robots: Optimal decoupling into sequential plans,'' in
  \emph{Robotics: Science and Systems}, 2009.

\bibitem{SolHal12}
K.~Solovey and D.~Halperin, ``$k$-color multi-robot motion planning,'' in
  \emph{Proceedings Workshop on Algorithmic Foundations of Robotics}, 2012.

\bibitem{YuLav13STAR}
J.~Yu and S.~M. LaValle, ``Multi-agent path planning and network flow,'' in
  \emph{Algorithmic Foundations of Robotics {X}, Springer Tracts in Advanced
  Robotics}.\hskip 1em plus 0.5em minus 0.4em\relax Springer Berlin/Heidelberg,
  2013, vol.~86, pp. 157--173.

\bibitem{TurMicKum14}
M.~Turpin, K.~Mohta, N.~Michael, and V.~Kumar, ``{CAPT}: Concurrent assignment
  and planning of trajectories for multiple robots,'' \emph{International
  Journal of Robotics Research}, vol.~33, no.~1, pp. 98--112, 2014.

\bibitem{ChoLynHutKanBurKavThr05}
H.~Choset, K.~M. Lynch, S.~Hutchinson, G.~Kantor, W.~Burgard, L.~E. Kavraki,
  and S.~Thrun, \emph{Principles of Robot Motion: Theory, Algorithms, and
  Implementations}.\hskip 1em plus 0.5em minus 0.4em\relax Cambridge, MA: MIT
  Press, 2005.

\bibitem{blm-rvo}
J.~van~den Berg, M.~C. Lin, and D.~Manocha, ``Reciprocal velocity obstacles for
  real-time multi-agent navigation,'' in \emph{Proceedings IEEE International
  Conference on Robotics \& Automation}, 2008, pp. 1928--1935.

\bibitem{branicky2006sampling}
M.~S. Branicky, M.~M. Curtiss, J.~Levine, and S.~Morgan, ``Sampling-based
  planning, control and verification of hybrid systems,'' \emph{IEE Proceedings
  Control Theory and Applications}, vol. 153, no.~5, p. 575, 2006.

\bibitem{khatib1986real}
O.~Khatib, ``Real-time obstacle avoidance for manipulators and mobile robots,''
  \emph{The international journal of robotics research}, vol.~5, no.~1, pp.
  90--98, 1986.

\bibitem{earl2005iterative}
M.~G. Earl and R.~D'Andrea, ``Iterative milp methods for vehicle-control
  problems,'' \emph{IEEE Transactions on Robotics}, vol.~21, no.~6, pp.
  1158--1167, 2005.

\bibitem{bekris2007decentralized}
K.~E. Bekris, K.~I. Tsianos, and L.~E. Kavraki, ``A decentralized planner that
  guarantees the safety of communicating vehicles with complex dynamics that
  replan online,'' in \emph{2007 IEEE/RSJ International Conference on
  Intelligent Robots and Systems}.\hskip 1em plus 0.5em minus 0.4em\relax IEEE,
  2007, pp. 3784--3790.

\bibitem{knepper2012pedestrian}
R.~A. Knepper and D.~Rus, ``Pedestrian-inspired sampling-based multi-robot
  collision avoidance,'' in \emph{2012 IEEE RO-MAN: The 21st IEEE International
  Symposium on Robot and Human Interactive Communication}.\hskip 1em plus 0.5em
  minus 0.4em\relax IEEE, 2012, pp. 94--100.

\bibitem{alonso2015local}
J.~Alonso-Mora, R.~Knepper, R.~Siegwart, and D.~Rus, ``Local motion planning
  for collaborative multi-robot manipulation of deformable objects,'' in
  \emph{2015 IEEE International Conference on Robotics and Automation
  (ICRA)}.\hskip 1em plus 0.5em minus 0.4em\relax IEEE, 2015, pp. 5495--5502.

\bibitem{HalLatWil00}
D.~Halperin, J.-C. Latombe, and R.~Wilson, ``A general framework for assembly
  planning: The motion space approach,'' \emph{Algorithmica}, vol.~26, no. 3-4,
  pp. 577--601, 2000.

\bibitem{Nna92}
B.~Nnaji, \emph{Theory of Automatic Robot Assembly and Programming}.\hskip 1em
  plus 0.5em minus 0.4em\relax Chapman \& Hall, 1992.

\bibitem{RodAma10}
S.~Rodriguez and N.~M. Amato, ``Behavior-based evacuation planning,'' in
  \emph{Proceedings IEEE International Conference on Robotics \& Automation},
  2010, pp. 350--355.

\bibitem{BalArk98}
T.~Balch and R.~C. Arkin, ``Behavior-based formation control for multirobot
  teams,'' \emph{IEEE Transactions on Robotics \& Automation}, vol.~14, no.~6,
  pp. 926--939, 1998.

\bibitem{PodSuk04}
S.~Poduri and G.~S. Sukhatme, ``Constrained coverage for mobile sensor
  networks,'' in \emph{Proceedings IEEE International Conference on Robotics \&
  Automation}, 2004.

\bibitem{ShuMurBen07}
B.~Shucker, T.~Murphey, and J.~K. Bennett, ``Switching rules for decentralized
  control with simple control laws,'' in \emph{American Control Conference},
  Jul 2007, pp. 1485--1492.

\bibitem{SmiEgeHow08}
B.~Smith, M.~Egerstedt, and A.~Howard, ``Automatic generation of persistent
  formations for multi-agent networks under range constraints,''
  \emph{ACM/Springer Mobile Networks and Applications Journal}, vol.~14, no.~3,
  pp. 322--335, June 2009.

\bibitem{TanPapKum04}
H.~Tanner, G.~Pappas, and V.~Kumar, ``Leader-to-formation stability,''
  \emph{IEEE Transactions on Robotics \& Automation}, vol.~20, no.~3, pp.
  443--455, Jun 2004.

\bibitem{FoxBurKruThr00}
D.~Fox, W.~Burgard, H.~Kruppa, and S.~Thrun, ``A probabilistic approach to
  collaborative multi-robot localization,'' \emph{Autonomous Robots}, vol.~8,
  no.~3, pp. 325--344, June 2000.

\bibitem{GriAke05}
E.~J. Griffith and S.~Akella, ``Coordinating multiple droplets in planar array
  digital microfluidic systems,'' \emph{International Journal of Robotics
  Research}, vol.~24, no.~11, pp. 933--949, 2005.

\bibitem{MatNilSim95}
M.~J. Matari\'c, M.~Nilsson, and K.~T. Simsarian, ``Cooperative multi-robot box
  pushing,'' in \emph{Proceedings IEEE/RSJ International Conference on
  Intelligent Robots \& Systems}, 1995, pp. 556--561.

\bibitem{RusDonJen95}
D.~Rus, B.~Donald, and J.~Jennings, ``Moving furniture with teams of autonomous
  robots,'' in \emph{Proceedings IEEE/RSJ International Conference on
  Intelligent Robots \& Systems}, 1995, pp. 235--242.

\bibitem{JenWheEva97}
J.~S. Jennings, G.~Whelan, and W.~F. Evans, ``Cooperative search and rescue
  with a team of mobile robots,'' in \emph{Proceedings IEEE International
  Conference on Robotics \& Automation}, 1997.

\bibitem{reif1985complexity}
J.~H. Reif, ``Complexity of the generalized mover's problem.'' DTIC Document,
  Tech. Rep., 1985.

\bibitem{Can88}
J.~F. Canny, \emph{The Complexity of Robot Motion Planning}.\hskip 1em plus
  0.5em minus 0.4em\relax Cambridge, MA: MIT Press, 1988.

\bibitem{SpiYak84}
P.~Spirakis and C.~K. Yap, ``Strong {NP}-hardness of moving many discs,''
  \emph{Information Processing Letters}, vol.~19, no.~1, pp. 55--59, 1984.

\bibitem{HopSchSha84}
J.~E. Hopcroft, J.~T. Schwartz, and M.~Sharir, ``On the complexity of motion
  planning for multiple independent objects; {PSPACE}-hardness of the
  ``warehouseman's problem'','' \emph{The International Journal of Robotics
  Research}, vol.~3, no.~4, pp. 76--88, 1984.

\bibitem{KloHut06}
S.~Kloder and S.~Hutchinson, ``Path planning for permutation-invariant
  multirobot formations,'' \emph{IEEE Transactions on Robotics}, vol.~22,
  no.~4, pp. 650--665, 2006.

\bibitem{HeaDem05}
R.~A. Hearn and E.~D. Demaine, ``{PSPACE}-completeness of sliding-block puzzles
  and other problems through the nondeterministic constraint logic model of
  computation,'' \emph{Theoretical Computer Science}, vol. 343, no.~1, pp.
  72--96, 2005.

\bibitem{SolHal15}
K.~Solovey and D.~Halperin, ``On the hardness of unlabeled multi-robot motion
  planning,'' in \emph{Robotics: Science and Systems (RSS)}, 2015.

\bibitem{KatYuLav13ICRA-C}
M.~Katsev, J.~Yu, and S.~M. LaValle, ``Efficient formation path planning on
  large graphs,'' in \emph{Proceedings IEEE International Conference on
  Robotics \& Automation}, 2013, pp. 3606--3611.

\bibitem{adler2015efficient}
A.~Adler, M.~De~Berg, D.~Halperin, and K.~Solovey, ``Efficient multi-robot
  motion planning for unlabeled discs in simple polygons,'' in
  \emph{Algorithmic Foundations of Robotics XI}.\hskip 1em plus 0.5em minus
  0.4em\relax Springer, 2015, pp. 1--17.

\bibitem{SolYu15}
K.~Solovey, J.~Yu, O.~Zamir, and D.~Halperin, ``Motion planning for unlabeled
  discs with optimality guarantees,'' in \emph{Robotics: Science and Systems},
  2015.

\bibitem{AulMonParPer99}
V.~Auletta, A.~Monti, M.~Parente, and P.~Persiano, ``A linear-time algorithm
  for the feasbility of pebble motion on trees,'' \emph{Algorithmica}, vol.~23,
  pp. 223--245, 1999.

\bibitem{GorHas10}
G.~Goraly and R.~Hassin, ``Multi-color pebble motion on graph,''
  \emph{Algorithmica}, vol.~58, pp. 610--636, 2010.

\bibitem{YuArxiv-1301-2342}
J.~Yu, ``A linear time algorithm for the feasibility of pebble motion on
  graphs,'' \emph{arXiv:1301.2342}, 2013.

\bibitem{YuRus15STAR}
J.~Yu and D.~Rus, ``Pebble motion on graphs with rotations: Efficient
  feasibility tests and planning,'' in \emph{Algorithmic Foundations of
  Robotics {XI}, Springer Tracts in Advanced Robotics}, vol. 107.\hskip 1em
  plus 0.5em minus 0.4em\relax Springer Berlin/Heidelberg, 2015, pp. 729--746.

\bibitem{RatWar90}
D.~Ratner and M.~Warmuth, ``The $(n^2-1)$-puzzle and related relocation
  problems,'' \emph{Journal of Symbolic Computation}, vol.~10, pp. 111--137,
  1990.

\bibitem{YuLav13AAAI}
J.~Yu and S.~M. LaValle, ``Structure and intractability of optimal multi-robot
  path planning on graphs,'' in \emph{Proceedings AAAI National Conference on
  Artificial Intelligence}, 2013, pp. 1444--1449.

\bibitem{alami1995multi}
R.~Alami, F.~Robert, F.~Ingrand, and S.~Suzuki, ``Multi-robot cooperation
  through incremental plan-merging,'' in \emph{Robotics and Automation, 1995.
  Proceedings., 1995 IEEE International Conference on}, vol.~3.\hskip 1em plus
  0.5em minus 0.4em\relax IEEE, 1995, pp. 2573--2579.

\bibitem{qutub1997solve}
S.~Qutub, R.~Alami, and F.~Ingrand, ``How to solve deadlock situations within
  the plan-merging paradigm for multi-robot cooperation,'' in \emph{Intelligent
  Robots and Systems, 1997. IROS'97., Proceedings of the 1997 IEEE/RSJ
  International Conference on}, vol.~3.\hskip 1em plus 0.5em minus 0.4em\relax
  IEEE, 1997, pp. 1610--1615.

\bibitem{saha2006multi}
M.~Saha and P.~Isto, ``Multi-robot motion planning by incremental
  coordination,'' in \emph{2006 IEEE/RSJ International Conference on
  Intelligent Robots and Systems}.\hskip 1em plus 0.5em minus 0.4em\relax IEEE,
  2006, pp. 5960--5963.

\bibitem{WagChoC11}
G.~Wagner and H.~Choset, ``M*: A complete multirobot path planning algorithm
  with performance bounds,'' in \emph{Proceedings IEEE/RSJ International
  Conference on Intelligent Robots \& Systems}, 2011, pp. 3260--3267.

\bibitem{ferner2013odrm}
C.~Ferner, G.~Wagner, and H.~Choset, ``Odrm* optimal multirobot path planning
  in low dimensional search spaces,'' in \emph{Robotics and Automation (ICRA),
  2013 IEEE International Conference on}.\hskip 1em plus 0.5em minus
  0.4em\relax IEEE, 2013, pp. 3854--3859.

\bibitem{ShaSteFelStu12}
G.~Sharon, R.~Stern, A.~Felner, and N.~Sturtevant, ``{Conflict-Based Search for
  Optimal Multi-Agent Path Finding},'' in \emph{Proc of the Twenty-Sixth AAAI
  Conference on Artificial Intelligence}, 2012.

\bibitem{sharon2013increasing}
G.~Sharon, R.~Stern, M.~Goldenberg, and A.~Felner, ``The increasing cost tree
  search for optimal multi-agent pathfinding,'' \emph{Artificial Intelligence},
  vol. 195, pp. 470--495, 2013.

\bibitem{boyarski2015icbs}
E.~Boyarski, A.~Felner, R.~Stern, G.~Sharon, O.~Betzalel, D.~Tolpin, and
  E.~Shimony, ``Icbs: The improved conflict-based search algorithm for
  multi-agent pathfinding,'' in \emph{Eighth Annual Symposium on Combinatorial
  Search}, 2015.

\bibitem{cohen2016improved}
L.~Cohen, T.~Uras, T.~Kumar, H.~Xu, N.~Ayanian, and S.~Koenig, ``Improved
  bounded-suboptimal multi-agent path finding solvers,'' in \emph{International
  Joint Conference on Artificial Intelligence}, 2016.

\bibitem{honig2016multi}
W.~H{\"o}nig, T.~S. Kumar, L.~Cohen, H.~Ma, H.~Xu, N.~Ayanian, and S.~Koenig,
  ``Multi-agent path finding with kinematic constraints.'' in \emph{ICAPS},
  2016, pp. 477--485.

\bibitem{Sur12}
P.~Surynek, ``Towards optimal cooperative path planning in hard setups through
  satisfiability solving,'' in \emph{Proceedings 12th Pacific Rim International
  Conference on Artificial Intelligence}, 2012.

\bibitem{erdem2013general}
E.~Erdem, D.~G. Kisa, U.~{\"O}ztok, and P.~Schueller, ``A general formal
  framework for pathfinding problems with multiple agents.'' in \emph{AAAI},
  2013.

\bibitem{YuLav16TRO}
J.~Yu and S.~M. LaValle, ``Optimal multi-robot path planning on graphs:
  Complete algorithms and effective heuristics,'' \emph{IEEE Transactions on
  Robotics}, vol.~32, no.~5, pp. 1163--1177, 2016.

\bibitem{ma2017lifelong}
H.~Ma, J.~Li, T.~Kumar, and S.~Koenig, ``Lifelong multi-agent path finding for
  online pickup and delivery tasks,'' in \emph{Proceedings of the 16th
  Conference on Autonomous Agents and MultiAgent Systems}.\hskip 1em plus 0.5em
  minus 0.4em\relax International Foundation for Autonomous Agents and
  Multiagent Systems, 2017, pp. 837--845.

\bibitem{atzmon2018robust}
D.~Atzmon, R.~Stern, A.~Felner, G.~Wagner, R.~Bart{\'a}k, and N.-F. Zhou,
  ``Robust multi-agent path finding,'' in \emph{Proceedings of the 17th
  International Conference on Autonomous Agents and MultiAgent Systems}.\hskip
  1em plus 0.5em minus 0.4em\relax International Foundation for Autonomous
  Agents and Multiagent Systems, 2018, pp. 1862--1864.

\bibitem{Kor84}
D.~M. Kornhauser, ``Coordinating pebble motion on graphs, the diameter of
  permutation groups, and applications,'' Ph.D. dissertation, Massachusetts
  Institute of Technology, 1984.

\bibitem{bollobas2013modern}
B.~Bollob{\'a}s, \emph{Modern graph theory}.\hskip 1em plus 0.5em minus
  0.4em\relax Springer Science \& Business Media, 2013, vol. 184.

\bibitem{Kuh55}
H.~W. Kuhn, ``The {Hungarian} method for the assignment problem,'' \emph{Naval
  Research Logistics Quarterly}, vol.~2, pp. 83--97, 1955.

\end{thebibliography}
\end{document}